\documentclass[letterpaper, 10 pt,conference]{ieeeconf}
\IEEEoverridecommandlockouts 
\usepackage{cite}

\usepackage{amsmath,amssymb,amsfonts}
\usepackage{graphicx}
\usepackage{textcomp}

\DeclareMathOperator*{\argmin}{arg\,min}
\usepackage[table,xcdraw]{xcolor}
\usepackage{subcaption}
\usepackage{authblk}
\let\labelindent\relax
\usepackage{enumitem}
\usepackage{hyperref}
\usepackage{breqn}
\usepackage{algpseudocode}
\usepackage[linesnumbered,ruled,lined]{algorithm2e}
\usepackage{mdframed}
\usepackage{array}
\usepackage{makecell}
\usepackage{tcolorbox}
\usepackage{svg}

\newtheorem{proposition}{Proposition}

\allowdisplaybreaks

\begin{document}

\title{Bi-level Trajectory Optimization on Uneven Terrains with Differentiable Wheel-Terrain Interaction Model.
}

\author{
Amith Manoharan$^{1}$,
Aditya Sharma$^{2}$,
Himani Belsare$^{2}$,
Kaustab Pal$^{2}$,\\
K. Madhava Krishna$^{2}$, and
Arun Kumar Singh$^{1}$

\thanks{$^1$ are with the Institute of Technology, University of
Tartu, Tartu, Estonia. }
\thanks{$^2$ are with RRC, IIIT Hyderabad, India.}
\thanks{This work was co-funded by the European Social Fund and Estonian Research Council via project TEM-TA101, Grant PSG753, and in part by grants available through Mathworks India.
}
}

\maketitle

\begin{abstract}
Navigation of wheeled vehicles on uneven terrain necessitates going beyond the 2D approaches for trajectory planning. Specifically, it is essential to incorporate the full $6dof$ variation of vehicle pose and its associated stability cost in the planning process. To this end, most recent works aim to learn a neural network model to predict vehicle evolution. However, such approaches are data-intensive and fraught with generalization issues.

In this paper, we present a purely model-based approach that just requires the digital elevation information of the terrain. Specifically, we express the wheel-terrain interaction and $6dof$ pose prediction as a non-linear least squares (NLS) problem. As a result, trajectory planning can be viewed as a bi-level optimization. The inner optimization layer predicts the pose on the terrain along a given trajectory, while the outer layer deforms the trajectory itself to reduce the stability and kinematic costs of the pose. 

We improve the state-of-the-art in the following respects. First, we show that our NLS-based pose prediction closely matches the output of a high-fidelity physics engine. This result, coupled with the fact that we can query gradients of the NLS solver, makes our pose predictor a differentiable wheel-terrain interaction model. We further leverage this differentiability to efficiently solve the proposed bi-level trajectory optimization problem. Finally, we perform extensive experiments and comparisons with a baseline to showcase the effectiveness of our approach in obtaining smooth, stable trajectories.

\end{abstract}


\section{Introduction}
Applications like forestry, construction, and search and rescue require wheeled vehicles to navigate over uneven terrains. Thus, for automation in these use cases, we need a trajectory planner that can account for the vehicle tip-over stability over uneven terrains. This, in turn, requires explicitly predicting the vehicle's $6dof$ pose on the terrain and the kinematic and stability costs associated with those poses.

There are broadly two ways to predict the pose of a wheeled vehicle on an uneven terrain. The first approach is to use high-fidelity physics simulators that can model complex wheel-terrain interaction \cite{wiberg2021control}, often in the form of complex differential equations \cite{gattupalli2013simulation}. However, these simulators are often too slow to be used during online planning. Alternately, wheel-terrain interaction and pose prediction can be learned from recorded vehicle motions on different terrains \cite{agishev2022trajectory}. However, purely data-driven approaches are data-hungry and often struggle to generalize to novel scenarios.

Our main motivation in this paper is to make the physics/model-based approach more computationally tractable by finding the appropriate simplified abstraction for the wheel-terrain interaction. Moreover, we aim to make the interaction model differentiable and parallelizable to ensure online planning. We achieve these features through some core algorithmic innovations, which are summarized below, along with their benefits.

\begin{figure}[!t] 
    \centering
    \includegraphics[width=1.0\linewidth]{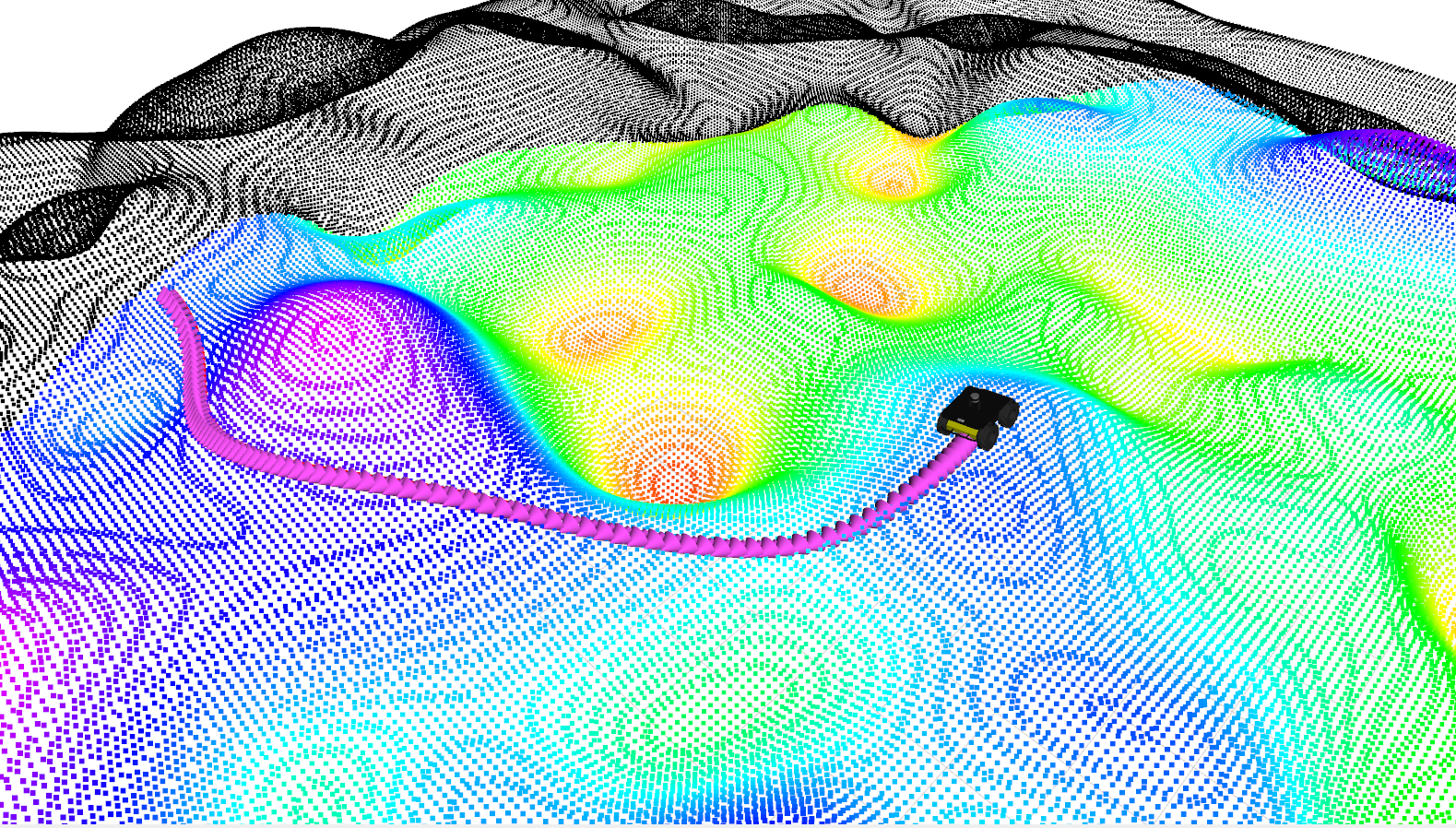}
    \caption{\footnotesize Husky navigating an uneven terrain. The black points represent the entire terrain, whereas the coloured points represent the terrain's local patch. All points within a given radius from the Husky are used to model the terrain, which is used to compute a stable trajectory to the goal. As seen above, the trajectory (shown in pink) successfully avoids the ditch and goes around it to reach the goal.}
    \vspace{-0.2mm}
    \label{fig:teaser_image}
    \vspace{-0.8cm}
\end{figure}

\noindent \textbf{Algorithmic Contribution:} We jointly view the wheeled vehicle and the underneath terrain through the lens of closed-loop kinematics, typically used for modeling parallel manipulators. This, in turn, leads to a set of coupled non-linear equations that dictate how the vehicle's $6dof$ pose depends on the terrain surface geometry. The pose prediction can then be obtained by minimizing the residual of the non-linear equations, which, in our case, is done through a least-squares optimization problem. Importantly, our set-up allows for differentiating through the least squares solver using implicit function theorem to ascertain how the pose of the vehicle will vary if its position on the terrain is perturbed. Thus, in essence, our non-linear least squares (NLS) approach provides a differentiable wheel-terrain interaction model, albeit at the kinematic level. 

The differentiability feature also opens up efficient ways of performing trajectory planning on uneven terrain. Specifically, \textit{for the first time}, we show that the planning problem can be viewed as a bi-level optimization. Its inner layer (NLS) predicts pose along a given trajectory while the outer layer perturbs the trajectory itself to lower the kinematic and stability costs. The implicit gradients from the inner layer can be used within any momentum-based gradient descent approach to efficiently solve the bi-level problem.

\noindent \textbf{State-of-the-Art Performance:} We show that the prediction from our NLS-based approach closely matches the output from a high-fidelity physics engine. As a result, the former provides a computationally cheap alternative for learning-based approaches that require tele-operating vehicles on potentially dangerous terrains. We show that by modeling the stability costs through force-angle measure \cite{papadopoulos2000force}, our trajectory optimizer produces trajectories that successfully avoid risky areas such as ditches and valleys. Finally, we show that our gradient-based approach is competitive with more computationally demanding sampling-based optimization, such as the Cross Entropy Method \cite{rubinstein1999cross} that does not require access to implicit gradients.

\section{Related Works}

\subsubsection*{Learning Based Wheel-Terrain Interaction}
An unsupervised method for the classification of a ground rover's slip events based on proprioceptive signals is proposed in \cite{bouguelia2017unsupervised}. The method is able to automatically discover and track various degrees of slip on the given terrain.
A meta-learning-based approach is proposed in \cite{banerjee2020adaptive} to predict future dynamics by leveraging rover-terrain interaction data. In \cite{ugenti2021learning}, an approach for learning and predicting the motion resistance encountered by a vehicle while traversing is studied using the information obtained from a stereovision device. A learning-based approach for predicting the pitch and roll angle from a depth map was proposed in \cite{datar2023learning}. Authors in \cite{vsalansky2021pose} show that purely data-driven approaches struggle to accurately predict the vehicle pose on the terrain. Thus, they embed physics-based priors into the learning pipeline. Our current approach is closely related \cite{vsalansky2021pose}, as our NLS-based pose predictor can also be potentially used as a physics-based prior.

\subsubsection*{Motion Planning on Uneven Terrains}
Model-based planning on uneven terrains requires a pose predictor module. This can be either a learned model as used in \cite{agishev2022trajectory}, or that obtained from first principles of physics and differential equations \cite{eathakota2011quasi}. A hybrid trajectory optimization method for generating stable and smooth trajectories for articulated tracked vehicles is proposed in \cite{xu2023hybrid}. A vehicle-terrain contact model is developed to divide the vehicle's motion into hybrid modes of driving and traversing, which is converted into a nonlinear programming problem to be solved in a moving-horizon manner. A terrain pose mapping to describe the impact of terrain on the vehicle is proposed in \cite{xu2023efficient}, based on which a trajectory optimization framework for car-like vehicles on uneven terrain is built.
Our proposed method differs from these cited works in how we leverage implicit gradients from a first principle-based wheel-terrain interaction model for trajectory optimization.

\subsubsection*{End-to-End Learning Based Approaches}
A learning-based end-to-end navigation for planetary rovers is given in \cite{feng2023learning}. A deep reinforcement learning-based navigation method is proposed to autonomously guide the rover towards goals through paths with low wheel slip ratios. An end-to-end network architecture is proposed, in which the visual perception and the wheel-terrain interaction are combined to learn the terrain properties implicitly. In \cite{hu2021sim}, a sim-to-real pipeline for a vehicle to learn how to navigate on rough terrains is presented. A deep reinforcement learning architecture is used to learn a navigation policy from data obtained from simulated environments. An approach that employs a fully-trained deep reinforcement learning network that uses elevation maps of the environment, vehicle pose, and goal as inputs to compute an attention mask of the environment is proposed in \cite{weerakoon2022terp}. The attention mask is then used to identify reduced stability regions in the terrain.

\subsubsection*{Connection to Author's Prior Work:} The proposed NLS-based pose predictor builds upon and improves our prior work \cite{singh2016feasible}. Herein, small angle approximation and local planar assumption on the underlying terrain were used to obtain closed-form solutions for pose prediction. However, such an approach is not appropriate for highly uneven terrains.  
\section{Preliminaries}

\begin{figure*}
\begin{subfigure}{7cm}
  \includegraphics[width=\linewidth]{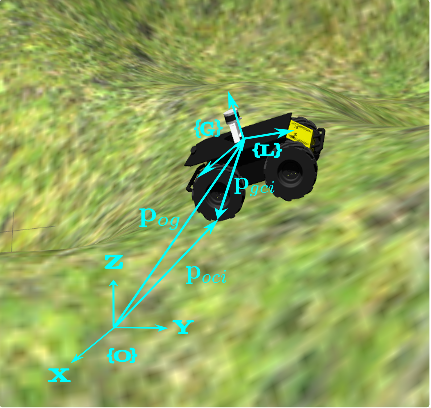}  
  \vspace{-5mm}
  \caption{}
  \label{fig:husky}
  \vspace{-1mm}
\end{subfigure}\hspace{20mm}
\begin{subfigure}{7cm}
  \includegraphics[width=7cm,height=7cm]{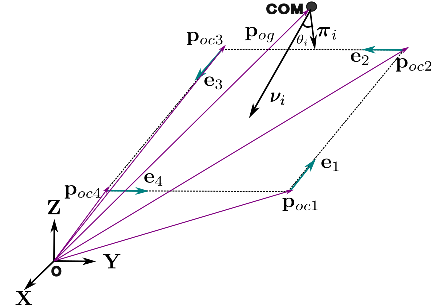}  
  \vspace{-5mm}
  \caption{}
  \label{fig:stability_diagram}
  \vspace{-1mm}
\end{subfigure}
\caption{\footnotesize (a) A four-wheeled vehicle with the geometry vectors describing the holonomic constraints. (b) Definition of the vectors associated with tip-over stability.}
\label{fig:terrain_contact_geometry}
\end{figure*}

\subsubsection*{Symbols and Notation} Scalars are denoted by normal lowercase letters, whereas bold font variants are used for vectors. Matrices are represented by uppercase bold font letters. The superscript $T$ is used to define the transpose of a matrix or a vector.

\subsection{Primer on Chain Rule of Differentiation}
\noindent This section covers some basic notations from multi-variable calculus that will be crucial for developing the implicit differentiation rules later in the paper. Most of the developments follow from the work \cite{gould2021deep}.

\noindent Consider a function $\psi(\lambda, \phi, \mu)$ where $\phi$ and $\mu$ are themselves functions of $\lambda$. In such a scenario, the chain rule of differentiation takes the following form.
\begin{equation}
    \frac{\mathrm{d}}{\mathrm{d}\lambda}\psi = \frac{\partial \psi}{\partial \lambda}\frac{\mathrm{d}\lambda}{\mathrm{d}\lambda} + \frac{\partial \psi}{\partial \phi}\frac{\mathrm{d}\phi}{\mathrm{d}\lambda} + \frac{\partial \psi}{\partial \mu}\frac{\mathrm{d}\mu}{\mathrm{d}\lambda}. \label{eqn:chain_rule}
\end{equation}
For a scalar function over multiple variables, $\psi (\lambda_1, \lambda_2, \dots \lambda_n) : \mathbb{R}^n \rightarrow \mathbb{R}$, the derivative vector can be written as 
\begin{equation}
    \mathrm{D}\psi = \begin{bmatrix}
        \frac{\partial \psi}{\partial \lambda_1},\frac{\partial \psi}{\partial \lambda_2},\dots,\frac{\partial \psi}{\partial \lambda_n} 
    \end{bmatrix} \in \mathbb{R}^{1 \times n}.
\end{equation}
Analogously,  for vector-valued functions $\boldsymbol{\psi} : \mathbb{R} \rightarrow \mathbb{R}^m$, defined over a scalar argument, we can write
\begin{equation}
    \mathrm{D}\psi = \begin{bmatrix}
        \frac{\mathrm{d} \psi_1}{\mathrm{d} \lambda},\dots,\frac{\mathrm{d} \psi_m}{\mathrm{d} \lambda} 
    \end{bmatrix} \in \mathbb{R}^{m \times 1}.
\end{equation}
The general definition of $\mathrm{D}\boldsymbol{\psi}$ of $\boldsymbol{\psi} : \mathbb{R}^n \rightarrow \mathbb{R}^m$ can be written as an $m \times n$ matrix with entries
\begin{equation}
    \left(\mathrm{D} \boldsymbol{\psi}(\lambda) \right)_{ij} = \frac{\partial \boldsymbol{\psi}_i}{\partial \lambda_j}(\lambda).
\end{equation}
Now, the chain rule for $h(\lambda) = \zeta(\psi(\lambda))$ can be written as
\begin{equation}
    \mathrm{D}h(\lambda) = \mathrm{D}\zeta(\psi(\lambda))\mathrm{D}\psi(\lambda).
\end{equation}
For the partial derivatives, the variable over which the derivative is computed is written as a subscript. For example, $\mathrm{D}_{\lambda} \psi(\lambda,\phi)$ denotes partial derivative with respect to $\lambda$. Similarly, $\mathrm{D}_{\lambda \phi}^2\psi$ means $\mathrm{D}_{\lambda}(\mathrm{D}_{\phi}\psi)^T$.

Finally, when subscripts are not given for multi-variate functions, the operator $\mathrm{D}$ means the total derivative with respect to the independent variables. So, with $\lambda$ as the independent variable, the vector form of Equation \eqref{eqn:chain_rule} is represented by
\begin{equation}
    \mathrm{D}\psi = \mathrm{D}_{\lambda}\psi + \mathrm{D}_{\phi}\psi\mathrm{D}\phi + \mathrm{D}_{\mu}\psi\mathrm{D}\mu. 
\end{equation}

\subsection{Trajectory Parametrization}\label{sec:traj_parametrization}

\noindent We assume that the vehicle follows a bi-cycle model. Moreover, we can leverage its differential flatness property \cite{han2023efficient} to plan directly in the positional space ($x_k, y_k$). The control inputs, for example, linear velocity and curvature, can then be obtained as a function of positional variables and their higher-order derivatives. We parametrize the position-level trajectory of the vehicle in terms of polynomials in the following form:

\begin{align}
    \begin{bmatrix}
        x_1, \dots, x_n 
    \end{bmatrix} = \mathbf{W}\mathbf{c}_{x},
     \begin{bmatrix}
        y_1, \dots, y_n 
    \end{bmatrix} = \mathbf{W}\mathbf{c}_{y},
    \label{param} 
\end{align}

\noindent where, $\mathbf{W}$ is a matrix formed with time-dependent polynomial basis functions and ($\mathbf{c}_{x}, \mathbf{c}_{y}$) are the coefficients of the polynomial. We can also express the derivatives in terms of $\dot{\mathbf{W}}, \ddot{\mathbf{W}}$.

\subsection{Tip-over stability criteria}
\noindent A vehicle's stability at a given point on the terrain can be calculated by the force angle measure or tip-over stability criteria proposed in \cite{rey1997online,papadopoulos2000force}. Fig.~\ref{fig:stability_diagram} shows the polygon formed with the ground contact points of a four-wheel vehicle. The edges of the polygon given by the vectors $\mathbf{e}_i$ represent the tip-over axes, i.e., the axes about which the vehicle could possibly tip-over. These axes can be computed in the following way:

\begin{align}
    \mathbf{e}_i &= \mathbf{p}_{oc,i+1} - \mathbf{p}_{oc,i}, \forall i = 1,2,3, \\
    \mathbf{e}_4 &= \mathbf{p}_{oc,1} - \mathbf{p}_{oc,4}. \nonumber
\end{align}

We define the unit vectors $\mathbf{\hat{e}}_i = \frac{\mathbf{e}_i}{\lVert \mathbf{e}_i \rVert}$, $\boldsymbol{\hat{\pi}}_i = \frac{\boldsymbol{\pi}_i}{\lVert \boldsymbol{\pi}_i \rVert}$, and $\boldsymbol{\hat{\nu}}_i = \frac{\boldsymbol{\nu}_i}{\lVert \boldsymbol{\nu}_i \rVert}$, where $\boldsymbol{\pi}_i$ are the tip-over axis normals which intersect the vehicle's center of mass (COM) and $\boldsymbol{\nu}_i$ are the components of the force acting on the vehicle which contribute to the tip-over. These vectors can be calculated as:

\begin{align}
    \boldsymbol{\pi}_i &= (\mathbf{I} - \mathbf{\hat{e}}_i\mathbf{\hat{e}}_i^T)(\mathbf{p}_{oc,i+1}-\mathbf{p}_{og}), \\
    \boldsymbol{\nu}_{i} &= (\mathbf{I} - \mathbf{\hat{e}}_i\mathbf{\hat{e}}_i^T)\boldsymbol{\nu},
\end{align}

\noindent where $\mathbf{I}$ is the identity matrix and $\boldsymbol{\nu}$ is the total force acting on the vehicle.
The angles between the forces and the tip-over axis normals can be computed as:

\begin{equation}
    \theta_i = \sigma_i \arccos{(\boldsymbol{\hat{\nu}}_i \cdot \boldsymbol{\hat{\pi}}_i)},
\end{equation}
where,
\begin{equation}
    \sigma_i = \begin{cases}
        +1, & (\boldsymbol{\hat{\pi}}_i \times \boldsymbol{\hat{\nu}}_i) \cdot \mathbf{\hat{e}}_i < 0 \\
        -1, & \text{otherwise},
    \end{cases}
\end{equation}
from which the tip-over stability criteria can be defined as:
\begin{equation}
    \theta_i > 0 \implies \min \theta_i > 0, \forall i = 1,2,3,4. \label{eqn:tip_over_stability_criteria} 
\end{equation}
We define the costs associated with the tip-over stability criteria as:
\begin{align}
    \theta_{c,i} &= \max(0,-\theta_i+\epsilon),\\
    \theta_{\Delta c,i} &= (\theta_i - \theta_{i+1})^2,
\end{align}
where $\epsilon$ is a small number. The total stability cost for the vehicle is defined as:
\begin{equation}
    c_s = \sum_i(\theta_{c,i} + w_{\theta} \, \theta_{\Delta c,i}), \label{eqn:stability_cost}
\end{equation}
where $w_{\theta}$ is a scaling factor.


\section{Main Algorithmic Results}
Our trajectory planning pipeline has the following components. First, we fit a functional form to the terrain digital elevation data, which is then used to formulate a differentiable wheel-terrain interaction model. Subsequently, we formulate a trajectory planning problem that leverages the differentiability of the wheel-terrain interaction.

\subsection{Functional Form for the Terrain Model}

Imagine that we are given the digital elevation data, which describes the height $z^j$ of the terrain at any given point $(x^j, y^j)$. For example, these can be obtained online from the point clouds obtained with on-board LiDAR. In this section, our aim is to fit a functional form to this data. That is, we want to obtain analytical relationships of the form:
\begin{equation}
    z^j = f(x^j, y^j). \label{eqn:fft_fn}
\end{equation}

Our approach is built on approximating $f$ in terms of Fourier basis functions. That is, 

\begin{dmath}
    f(x^j,y^j) = \sum_{n=1}^{N} a_n\cos{(\omega_{1,n} x^j +\omega_{2,n} y^j  )} + b_n\sin{(\omega_{3,n} x^j +\omega_{4,n} y^j  )}, 
\end{dmath}

\noindent where $\omega_{1,n}, \omega_{2,n}, \omega_{3,n}, \omega_{4,n}$ are the frequencies of the Fourier basis functions, $N$ is the number of frequencies, and $a_n, b_n$ are the weights associated with each functions. These are obtained by solving the following regression problem

\small
\begin{align}
    \sum_{j= 1}^{M}\Vert f(x^j,y^j)-z^j\Vert_2^2,
\end{align}
\normalsize

\noindent where $M$ is the total number of points in the digital elevation data or point-cloud. Fig.\ref{fig:terrain_patch} shows the fit obtained over circular terrain patches of radius $7m$.

\noindent \textbf{Practical Considerations:} The number of frequencies needed to get a good fit depends on the area of the terrain. Our implementation, which leverages warm-started gradient descent on GPUs, takes on an average of $0.5s$ to fit the Fourier approximation to a circular patch of radius $7m$. During practical implementation, we can adopt a receding horizon approach, where the vehicle can fit the terrain model over a certain distance and then plan a stable trajectory over it. Such an approach would ensure that the terrain fit is called sparingly.



\begin{figure}
\begin{subfigure}{4.2cm}
  \includegraphics[width=1.1\linewidth]{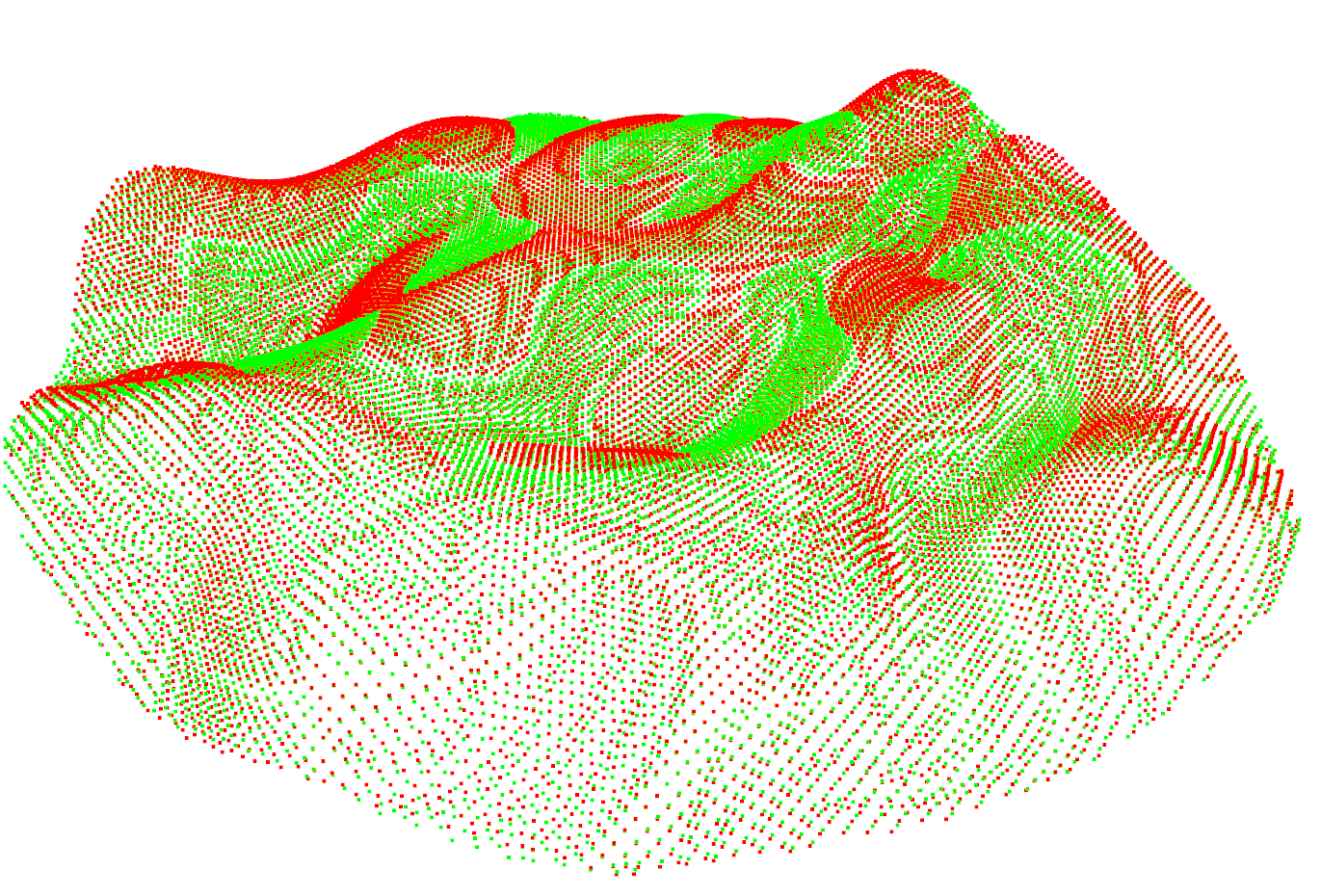}  
  \vspace{-5mm}
  \caption{}
  \label{fig:terrain_patch_6}
  \vspace{-1mm}
\end{subfigure}
\begin{subfigure}{4.2cm}
  \includegraphics[width=1.1\linewidth]{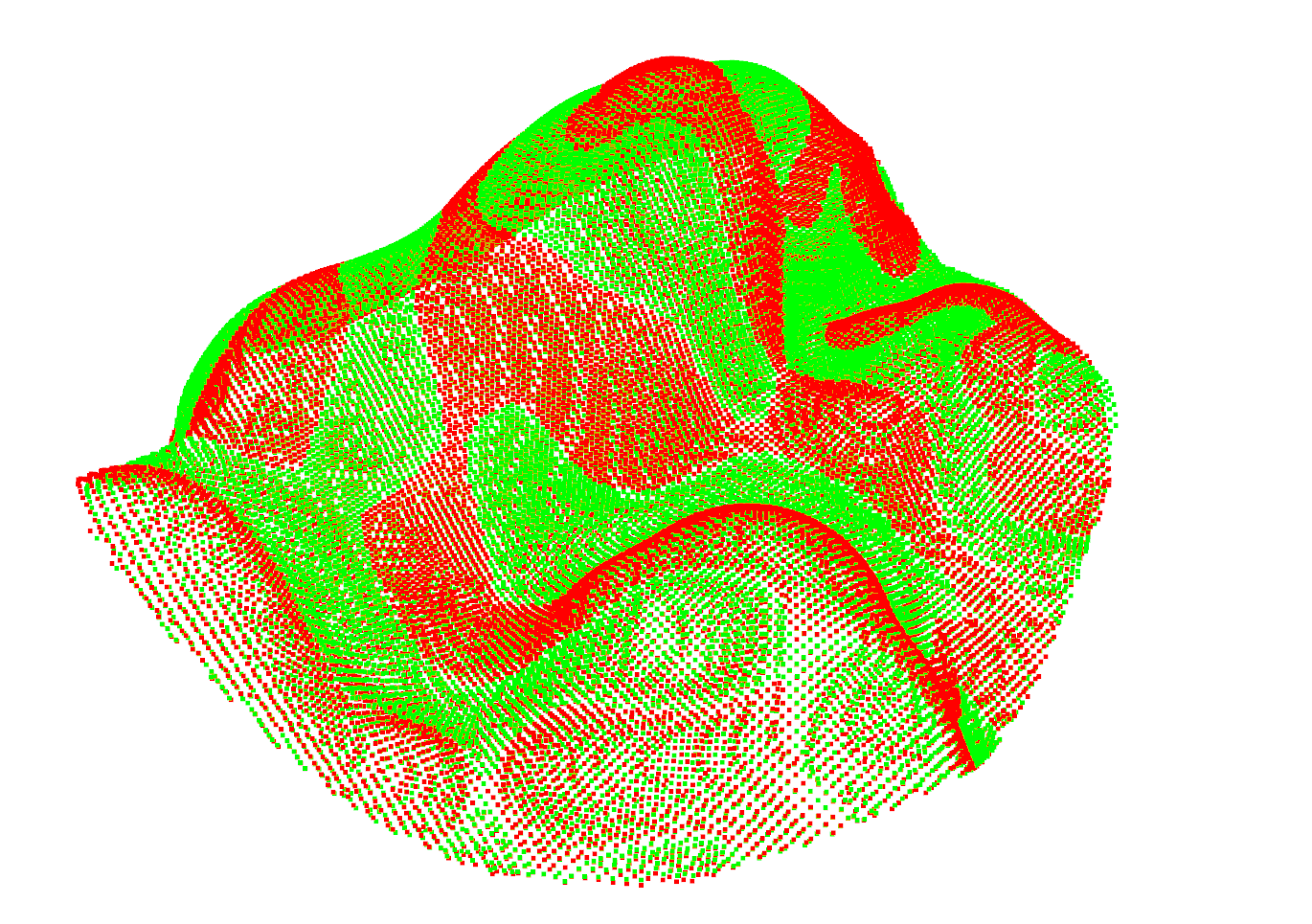}  
  \vspace{-5mm}
  \caption{}
  \label{fig:terrain_patch_7}
  \vspace{-1mm}
\end{subfigure}

\begin{subfigure}{4.2cm}
  \includegraphics[width=1.1\linewidth]{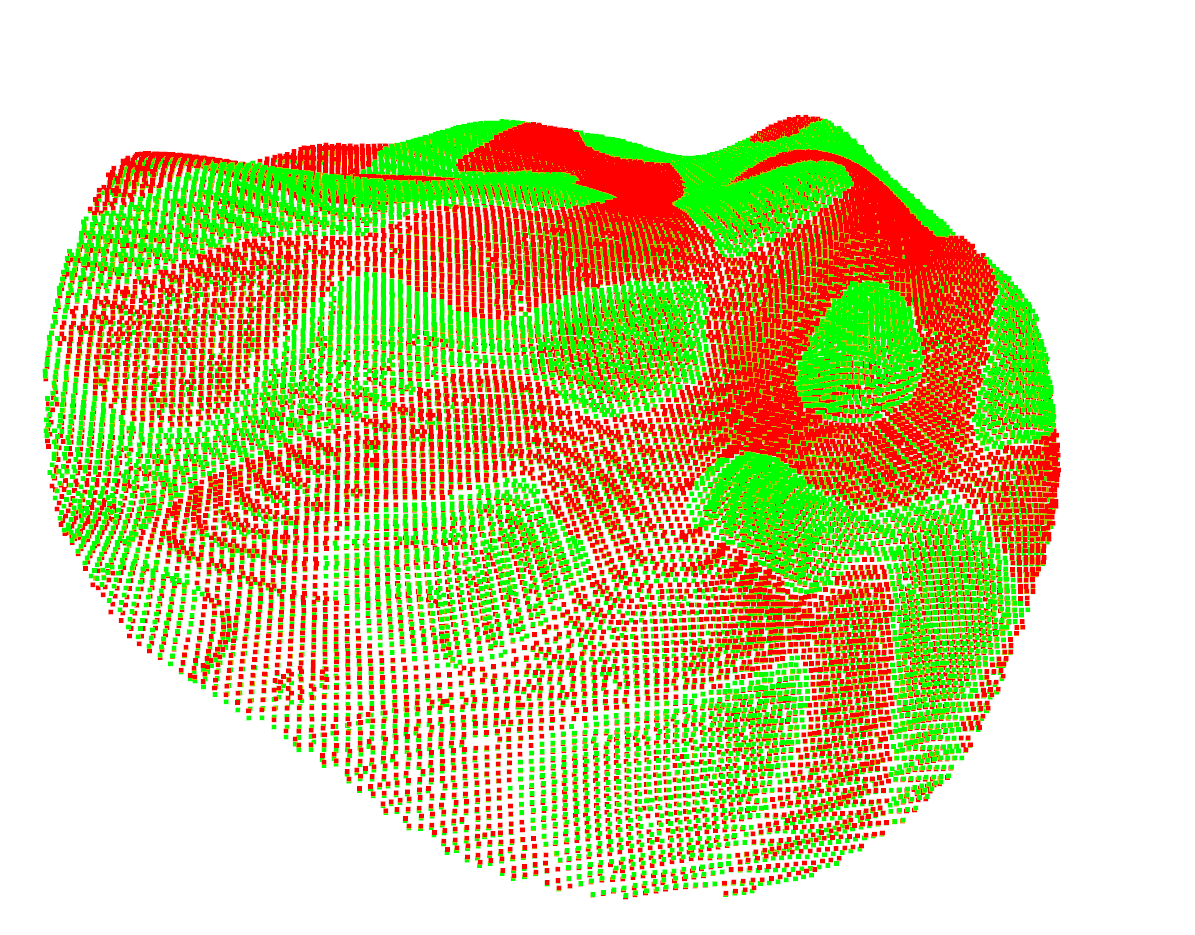}  
  \vspace{-5mm}
  \caption{}
  \label{fig:terrain_patch_8}
  \vspace{-1mm}
\end{subfigure}
\begin{subfigure}{4.2cm}
  \includegraphics[width=1.1\linewidth]{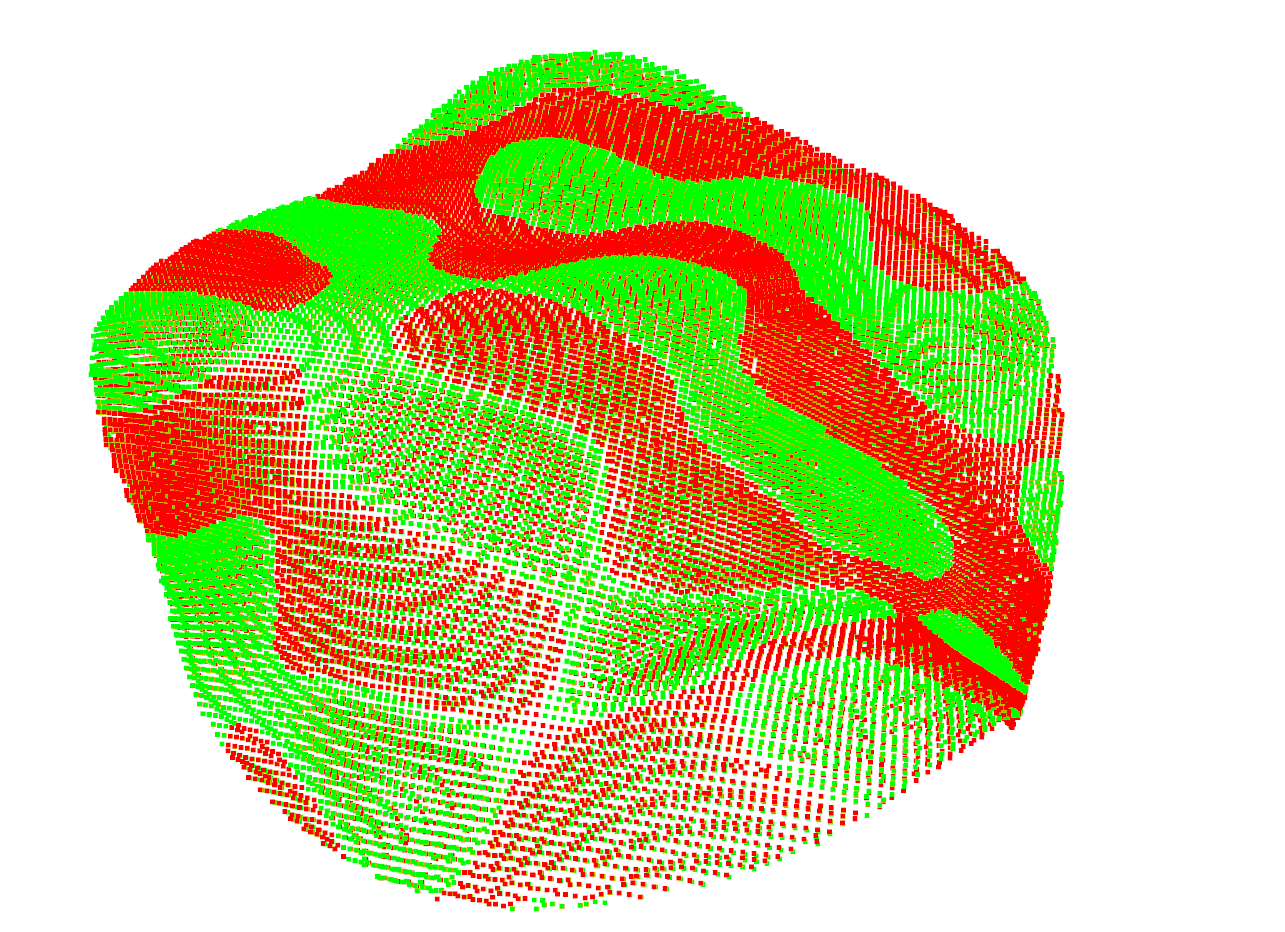}  
  \vspace{-5mm}
  \caption{}
  \label{fig:terrain_patch_9}
  \vspace{-1mm}
\end{subfigure}
\caption{\footnotesize Red points represent the ground truth terrain patch, and green points represent the terrain patch with predicted z coordinates.    }
\label{fig:terrain_patch}
\vspace{-0.8cm}
\end{figure}

\subsection{Differentiable Wheel-Terrain Interaction}


\noindent This section presents the first core result of the paper. Let the yaw plane states of the vehicle at time $k$ can be represented by the vector $\mathbf{x}_k = \begin{bmatrix}
    x_k & y_k & \alpha_k
\end{bmatrix}^T$, where $( x_k, y_k )$ is the position, and $\alpha_k$ is the heading angle of the vehicle. For wheeled vehicles with no-active suspension mechanism, the evolution of a vehicle's yaw plane
configuration, i.e., $\mathbf{x}_k$, can be directly controlled. The rest of the pose variables, i.e., $z_k$ coordinate, roll $\beta_k$, and pitch $\gamma_k$, will be a function of the yaw plane configuration and the terrain geometry. Mathematically, this dependency can be represented in the following manner.

\small
\begin{align}
    z_k = s_1(x_k,y_k,\alpha_k), \nonumber\\
    \beta_k = s_2(x_k,y_k,\alpha_k), \label{eqn:beta_fn}\\ 
    \gamma_k = s_3(x_k,y_k,\alpha_k). \nonumber 
\end{align}
\normalsize

\noindent In this section, we are interested in obtaining the functions $s_i$ that decide the wheel-terrain interaction. To this end, we view the wheeled vehicle on uneven terrain as a parallel manipulator \cite{chakraborty2004kinematics} and derive the so-called loop closure equations. To this end, refer to Fig.~\ref{fig:husky} and imagine that we have a global reference frame $ \{ O \} $ and a reference frame  $ \{ G \} $ that moves along with the vehicle and is attached to its center. The reference frame $ \{ L \} $ is also similar, but its orientation is the same as the vehicle. With these notations in place, the vector describing the wheel-ground contact point  $ \mathbf{p}_{oc, i}$ can be directly obtained from the terrain equation. Alternately, the contact points can also be reached by first following through $\mathbf{p}_{og}$ that identifies the center of the vehicle and then moving along $\mathbf{p}_{gc, i}$. This insight can be encoded into the following equation, which resembles the loop closure equations obtained for parallel manipulators \cite{chakraborty2004kinematics} \cite{singh2016feasible}.


\begin{equation}
    \mathbf{p}_{og} + \mathbf{p}_{gc, i} = \mathbf{p}_{oc, i}, \label{eqn:pos_vec}
\end{equation}

\begin{align}
    \mathbf{p}_{gc,i} &= \mathbf{R} \begin{bmatrix}
        \delta_i  h & r_i w &- (l_i)
    \end{bmatrix}, \forall i = 1,2,3,4, \label{pgci} \\ 
    \delta_i &= \begin{cases}
    1, & i= 1,4, \nonumber\\
    -1, & i = 2,3,
    \end{cases}\\
    r_i &= \frac{2.5-i}{|2.5-i|}, \nonumber \\
    \mathbf{p}_{og} &= \begin{bmatrix}
        x_k & y_k & z_k
    \end{bmatrix}, \label{pog}\\
    \mathbf{p}_{oc,i} &= \begin{bmatrix}
        x_{{c,i}_k} & y_{{c,i}_k} & z_{{c,i}_k} \label{poci}
    \end{bmatrix}.
\end{align}

\noindent The term $\mathbf{p}_{gc, i}$ is expressed as the vector describing the wheel contact point in the vehicle local frame $ \{ L \} $  multiplied with the rotation matrix between the $ \{ O \} $ and $ \{ L \} $ (refer Fig.~\ref{fig:husky}). The variables $r_i$ and $\delta_i$ have been incorporated to ensure the proper sign of $w$ and $h$ corresponding to each vertex of the chassis. $l_i$ are the equivalent leg lengths, including the wheels' radius. The constants $h$ and $w$ are the half-width and half-breadth of the chassis.



We also note that $\mathbf{p}_{oc,i}$ satisfy the terrain equation \eqref{eqn:fft_fn} and thus, it is possible to write:

\begin{equation}
    z_{{c,i}_k} = f(x_{{c,i}_k},y_{{c,i}_k}). \label{eqn:zci}
\end{equation}

We can expand  \eqref{eqn:pos_vec}  using the Euler-angle parametrization for the rotation matrix $\mathbf{R}$. When the resulting expressions are stacked alongside  \eqref{eqn:zci}, we obtain a set of 16 coupled non-linear equations in the following form: 
\begin{equation}
    g_j(\mathbf{x}_k,\mathbf{u}_k) = 0, \quad j=1\dots16,
\end{equation}
where $\mathbf{u}_k(\mathbf{x}_k)=\begin{bmatrix}
    z_k & \beta_k & \gamma_k & x_{{c,i}_k} & y_{{c,i}_k} & z_{{c,i}_k}
\end{bmatrix}^T$.
The pose prediction can now be formulated as a non-linear least squares (NLS) problem:

\begin{tcolorbox}[colback=green!5!white,colframe=green!75!black]
\begin{equation}
 \mathbf{u}_k^*(\mathbf{x}_k) = \argmin_{\mathbf{u}_k} \sum_j g_j(\mathbf{x}_k,\mathbf{u}_k)^2.  \label{eqn:u_k}
\end{equation}
\end{tcolorbox}

\newtheorem{remark}{Remark}\label{rem_1}

\begin{remark}
   The residual associated with the $\mathbf{u}_k^*(\mathbf{x}_k)$ can be used to ascertain whether it is possible for the wheeled vehicle to ensure contact on all four wheels on a particular patch of terrain.
\end{remark}

\noindent \textbf{Implicit Differentiation}
For an efficient solution of the trajectory optimization problem introduced later, it is important to compute the Jacobian of $\mathbf{u}_k^*(\mathbf{x}_k)$ with respect to its input parameter $\mathbf{x}_k$. However, the relationship between  $\mathbf{u}_k^*$ and $\mathbf{x}_k$ does not have an analytical form and thus requires tools from implicit differentiation. The derivation is based on Dini's implicit function theorem \cite{dontchev2009implicit} applied to the
first-order optimality condition. We define the following proposition that has been adapted from \cite{gould2021deep} for our pose prediction NLS.

\begin{proposition}
    Consider the NLS problem \eqref{eqn:u_k}. We can define $\mathbf{H}=\mathrm{D}_{\mathbf{u}_k\mathbf{u}_k}^2 \mathbf{g}(\mathbf{x}_k,\mathbf{u}_k(\mathbf{x}_k)) \in \mathbb{R}^{m \times m}$, where $\mathbf{g(.)}$ is obtained by stacking  $g_j(.)$ and $\mathbf{B}=\mathrm{D}_{\mathbf{x_ku_k}}^2\mathbf{g}(\mathbf{x}_k,\mathbf{u}_k(\mathbf{x}_k)) \in \mathbb{R}^{m \times n}$. Then, the Jacobian of the optimal pose is
    \begin{equation*}
        \mathrm{D}\mathbf{u}^*_k(\mathbf{x}_k) = -\mathbf{H}^{-1}\mathbf{B}, 
    \end{equation*}
   with the assumption that $\mathbf{H}$ is non-singular.
\end{proposition}
\begin{proof}
    The first-order optimality condition of our NLS is given by $\mathrm{D}_{\mathbf{u}_k}\mathbf{g}(\mathbf{x}_k,\mathbf{u}_k)=\mathbf{0}_{1 \times m}$. Differentiating both sides of this optimality condition with respect to $\mathbf{x}_k$ provides us the following equation
    \begin{align*}
       \mathbf{0}_{m \times n} &= \mathrm{D}(\mathrm{D}_{\mathbf{u}_k}\mathbf{g}(\mathbf{x}_k,\mathbf{u}_k))^T \\
        &=\mathrm{D}_{\mathbf{x}_k\mathbf{u}_k}^2\mathbf{g}(\mathbf{x}_k,\mathbf{u}_k)+\mathrm{D}_{\mathbf{u}_k\mathbf{u}_k}^2\mathbf{g}(\mathbf{x}_k,\mathbf{u}_k)\mathrm{D}\mathbf{u}_k(\mathbf{x}_k),     
    \end{align*}
    which can be rearranged to
    \begin{equation}
        \mathrm{D}\mathbf{u}_k(\mathbf{x}_k) = -(\mathrm{D}_{\mathbf{u}_k\mathbf{u}_k}^2\mathbf{g}(\mathbf{x}_k,\mathbf{u}_k))^{-1}\mathrm{D}_{\mathbf{x}_k\mathbf{u}_k}^2\mathbf{g}(\mathbf{x}_k,\mathbf{u}_k),
    \end{equation}
    when $\mathrm{D}_{\mathbf{u}_k\mathbf{u}_k}^2\mathbf{g}(\mathbf{x}_k,\mathbf{u}_k)$ is non-singular.
\end{proof}


\begin{figure}[h]
\begin{subfigure}{8cm}
    \includegraphics[width=1.1\linewidth]{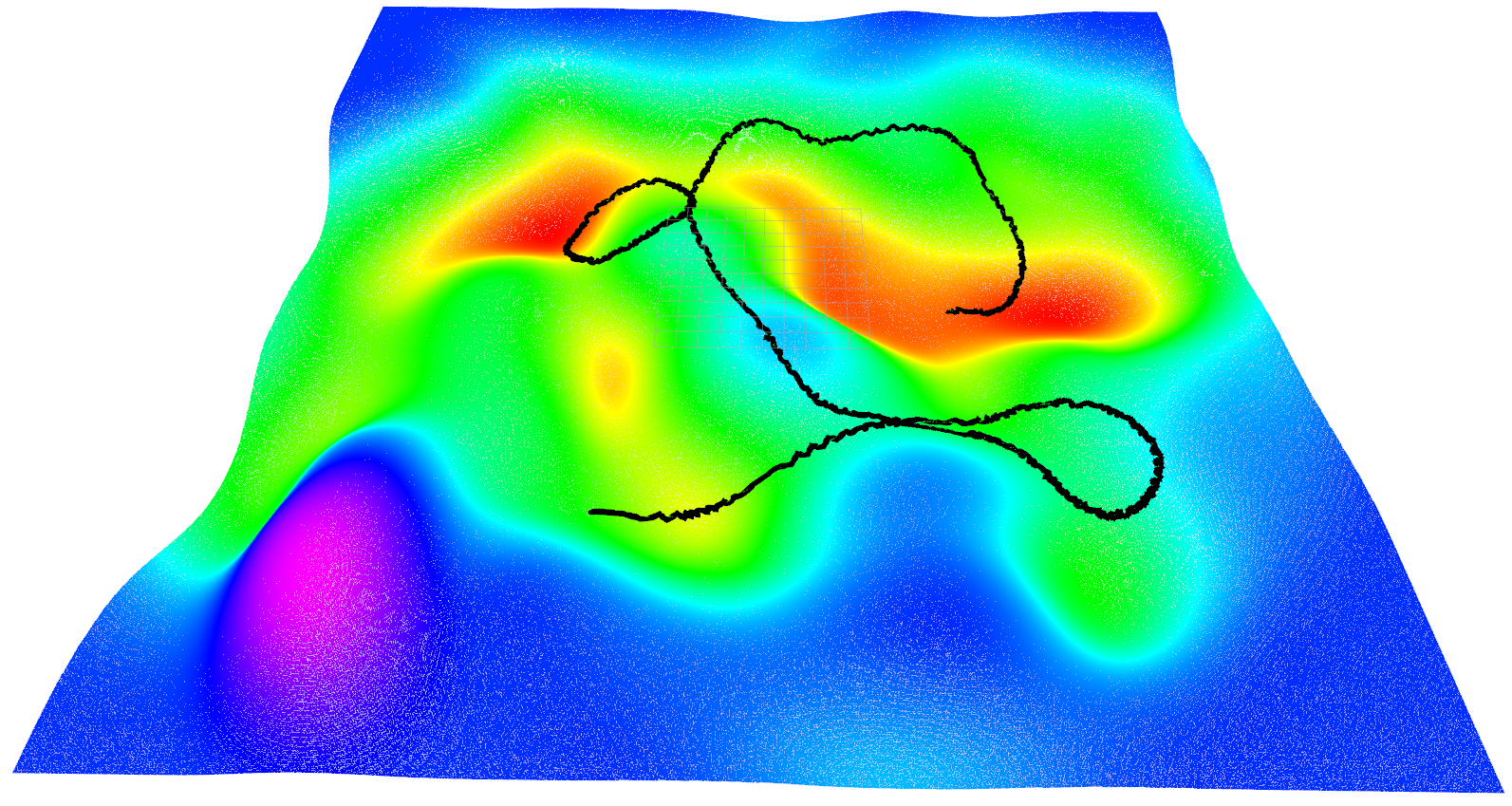}
    \caption{}
    \vspace{-0.9mm}
    \label{fig:manual_traj}
\end{subfigure}
\begin{subfigure}{4.2cm}
    \includegraphics[width=1.1\linewidth]{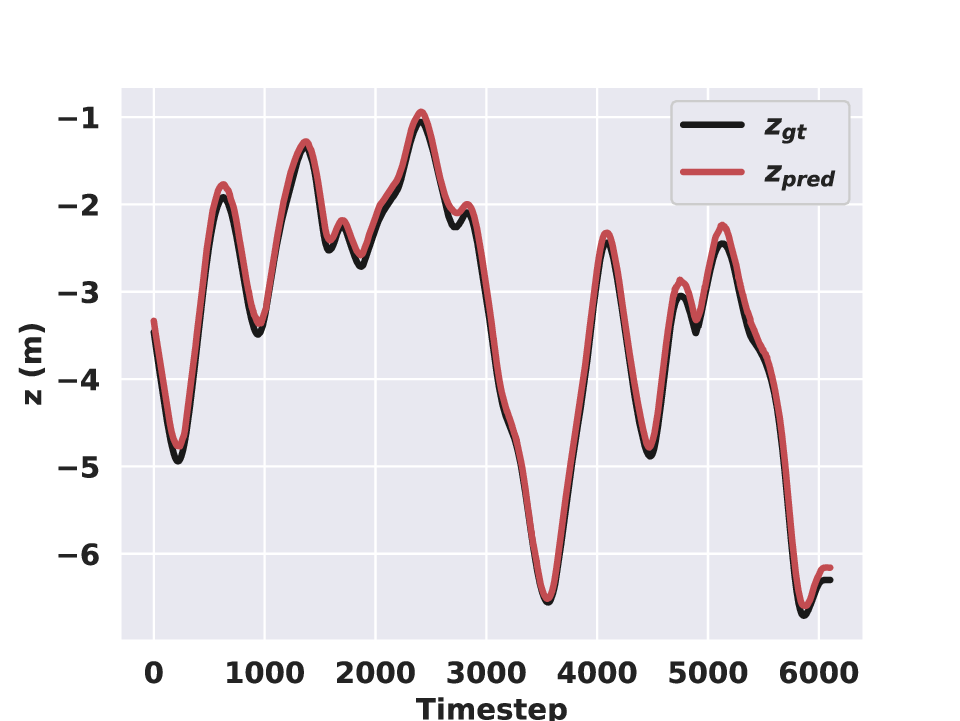}
    \caption{}
    \vspace{-0.9mm}
    \label{fig:z_plot_1}
\end{subfigure}
\begin{subfigure}{4.2cm}
    \includegraphics[width=1.1\linewidth,height=3.5cm]{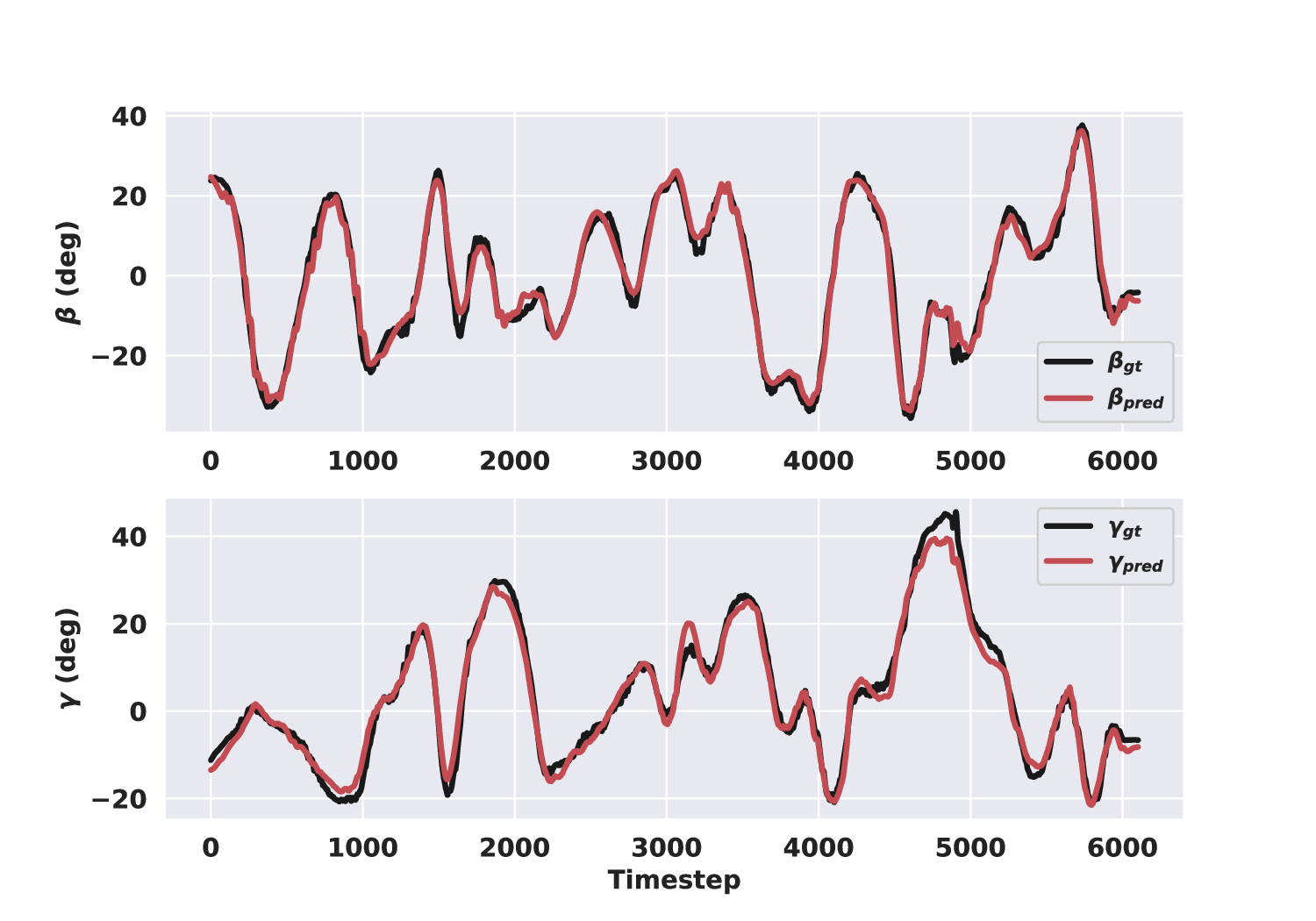}
    \caption{}
    \vspace{-0.9mm}
    \label{fig:beta_gamma_plot_1}
\end{subfigure}
\begin{subfigure}{4.2cm}
    \includegraphics[width=1.1\linewidth]{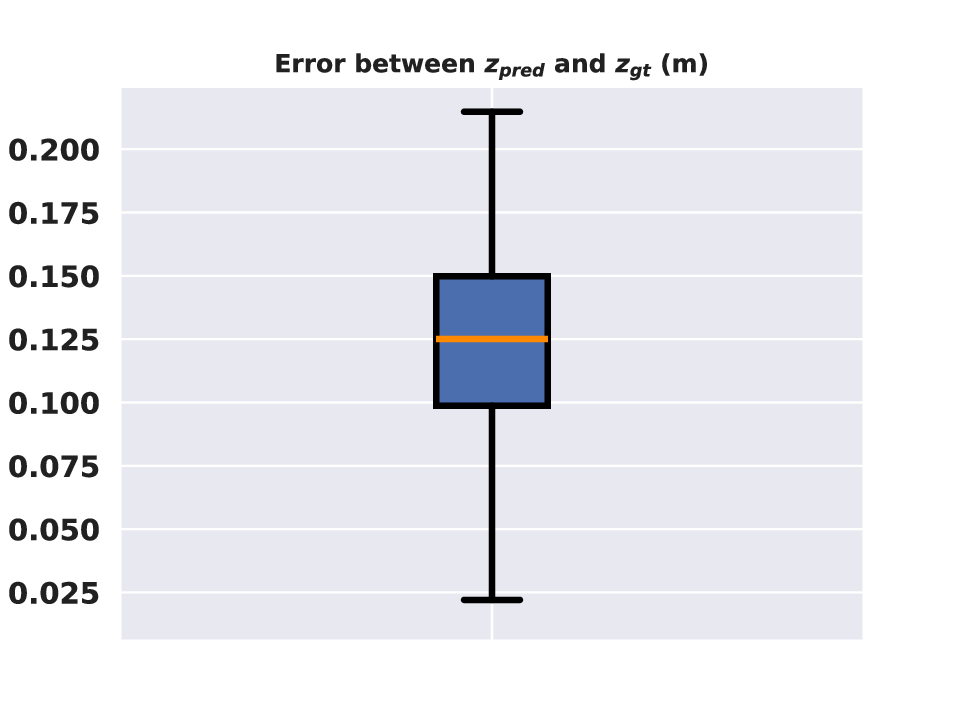}
    \caption{}
    \vspace{-2mm}
    \label{fig:z_box_plot_1}
\end{subfigure}
\begin{subfigure}{4.2cm}
    \includegraphics[width=1.1\linewidth]{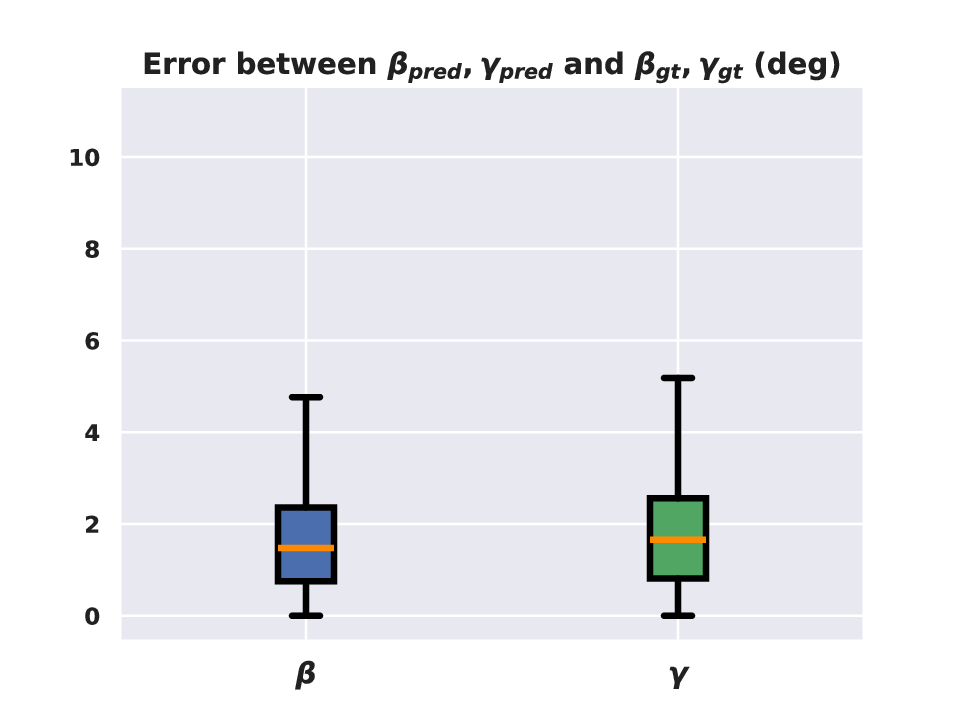}
    \caption{}
    \vspace{-2mm}
    \label{fig:beta_gamma_box_plot_1}
\end{subfigure}
\caption{\footnotesize (a) Trajectory obtained from manually driving Husky on a synthetic terrain in Gazebo. (b) $z$ ground truth and predicted values. (c) $\beta$ and $\gamma$ ground truth vs predicted. (d) Error statistics for $z$. (e) Error statistics for $\beta$ and $\gamma$. }
\end{figure}


\subsection{Bi-Level Optimization Based Trajectory Planning}

\noindent We assume a differentially flat non-holonomic vehicle. Thus, its heading, velocity, and curvature can be uniquely defined by the position-level trajectory. For example, the heading can be obtained as $\arctan(\dot{y}_k, \dot{x}_k)$. With this insight, we propose the following bi-level optimization problem

\small
\begin{subequations}
\begin{align}
    \min \sum_k &\left(c_r(\ddot{x}_k,\ddot{y}_k,\dot{x}_k,\dot{y}_k) + c_s(\mathbf{u}^*_k(\textbf{x}_k)) \right), \label{eqn:cost}\\
    &(x_0, y_0, \dot{x}_0, \dot{y}_0, \ddot{x}_0, \ddot{y}_0) = \mathbf{b}_0, \label{eqn:boundary_cond_initial} \\
    &(x_n, y_n, \dot{x}_n, \dot{y}_n, \ddot{x}_n, \ddot{y}_n) = \mathbf{b}_n, \label{eqn:final_boundary} \\
    &\underline{x} \leq x_k \leq \overline{x}, \\
    &\underline{y} \leq y_k \leq \overline{y}, \label{eqn:inequality}\\
    &\mathbf{u}_k^*(\mathbf{x}_k) = \argmin_{\mathbf{u}_k} \sum_j g_j(\mathbf{x}_k,\mathbf{u}_k)^2. \label{nls_biopt}
\end{align}
\begin{align}
    &c_a (\ddot{x}_k,\ddot{y}_k) = \ddot{x}_k^2+\ddot{y}_k^2, \\
    &c_c(\ddot{x}_k,\ddot{y}_k,\dot{x}_k,\dot{y}_k) = \ddot{y}\dot{x}-\ddot{x}\dot{y}/(\dot{x}^2+\dot{y}^2+\epsilon)^{\frac{3}{2}},\\
    &c_r(.)=c_a(.)+c_c(.).
\end{align}
\end{subequations}

\normalsize
\noindent The first term ($c_r(.)$) in the cost function  \eqref{eqn:cost} is the kinematics cost that ensures smoothness in the planned trajectory by penalizing high accelerations and sharp turns in the trajectory. The last term ($c_s(.)$) is the tip-over stability cost defined in \eqref{eqn:stability_cost}. Equality constraints \eqref{eqn:boundary_cond_initial} and \eqref{eqn:final_boundary} ensure that the planned trajectory satisfies the initial and final boundary conditions. Consequently, $\mathbf{b}_0,\mathbf{b}_n$ are the stacked vectors of initial and final positions, velocities, and accelerations. The constants $\underline{x},\underline{y},\overline{x},\overline{y}$ are the minimum and maximum bounds for $x_k,y_k$. The position level is critical for restricting the trajectory to lie within the patch that is currently observable by the vehicle's sensors.

Using the trajectory parametrization given in Sec~\ref{sec:traj_parametrization}, the Bi-level optimization problem can be written in a compact form as given below.

\begin{subequations}
\begin{align}
    \min_{\boldsymbol{\xi}} c_{r} (\boldsymbol{\xi})+c_s (\mathbf{u}^*(\boldsymbol{\xi})), \label{upper_cost}  \\
    \mathbf{u}(\boldsymbol{\xi}) = \min_{\mathbf{u}} \left\lVert \mathbf{g}(\boldsymbol{\xi}, \mathbf{u}) \right\rVert_{2}^2, \label{lower_cost}  \\
    \textbf{A}_{eq}\boldsymbol{\xi} = \textbf{b}_{eq},  \quad \textbf{A}\boldsymbol{\xi} \leq  \textbf{b}. \label{lower_eq} 
\end{align}
\end{subequations}

\noindent where $\boldsymbol{\xi} = [\mathbf{c}_x, \mathbf{c}_y]$ and $\mathbf{u}^*$ is formed by stacking $\mathbf{u}_k^*$ at different instants of time. The matrices $\mathbf{A}_{eq},\mathbf{A}$ and the vectors $\mathbf{b}_{eq},\mathbf{b}$ comes from \eqref{eqn:boundary_cond_initial}-\eqref{eqn:inequality}.


\subsection{Projected gradient descent}
\noindent The optimal value for the parameter $\boldsymbol{\xi}$ is obtained through projected gradient descent. The gradient descent update rule is given as
\begin{equation}
    ^{t+1}\overline{\boldsymbol{\xi}} =\, ^{t}\boldsymbol{\xi} - \eta\left(\nabla_{\boldsymbol{\xi}}c_r({^t}\boldsymbol{\xi})+ \mathrm{D}c_{s}\left(\mathbf{u}^{*}\left({^t}\boldsymbol{\xi}\right)\right)\right),
\end{equation}
where $\eta$ is the learning rate. The parameter $\boldsymbol{\xi}$ is projected onto the constraint set using the following Quadratic Programming problem at each iteration $t$ of the gradient descent.
\begin{align}
    ^{t+1}\boldsymbol{\xi} =&  \arg \min_{\boldsymbol{\xi}}\frac{1}{2}\left\lVert ^{t+1}\overline{\boldsymbol{\xi}} - ^{t+1}\boldsymbol{\xi} \right\rVert_2^2,\\ \label{eqn:projection}
    &\textbf{A}_{eq} \, ^{t+1}\boldsymbol{\xi} = \textbf{b}_{eq},\\
    &\textbf{A} \, ^{t+1}\boldsymbol{\xi} \leq  \textbf{b}. \label{eqn:proj_ineq}
\end{align}

\noindent The term $\mathrm{D}c_{s}\left(\mathbf{u}^{*}\left({^t}\boldsymbol{\xi}\right)\right)$ will be computed through implicit differentiation discussed in the previous section. However, since the implicit gradient is computed for a scalar cost $c_s$, we never actually have to explicitly form the Jacobian of $\textbf{u}^*$ with respect to $\boldsymbol{\xi}$. Rather, we just need the Jacobian-vector product that can be efficiently computed through existing autodifferentiaton libraries like JAX \cite{jax}. 

\begin{figure*}[h]
\begin{subfigure}{7cm}
  \includegraphics[width=\linewidth]{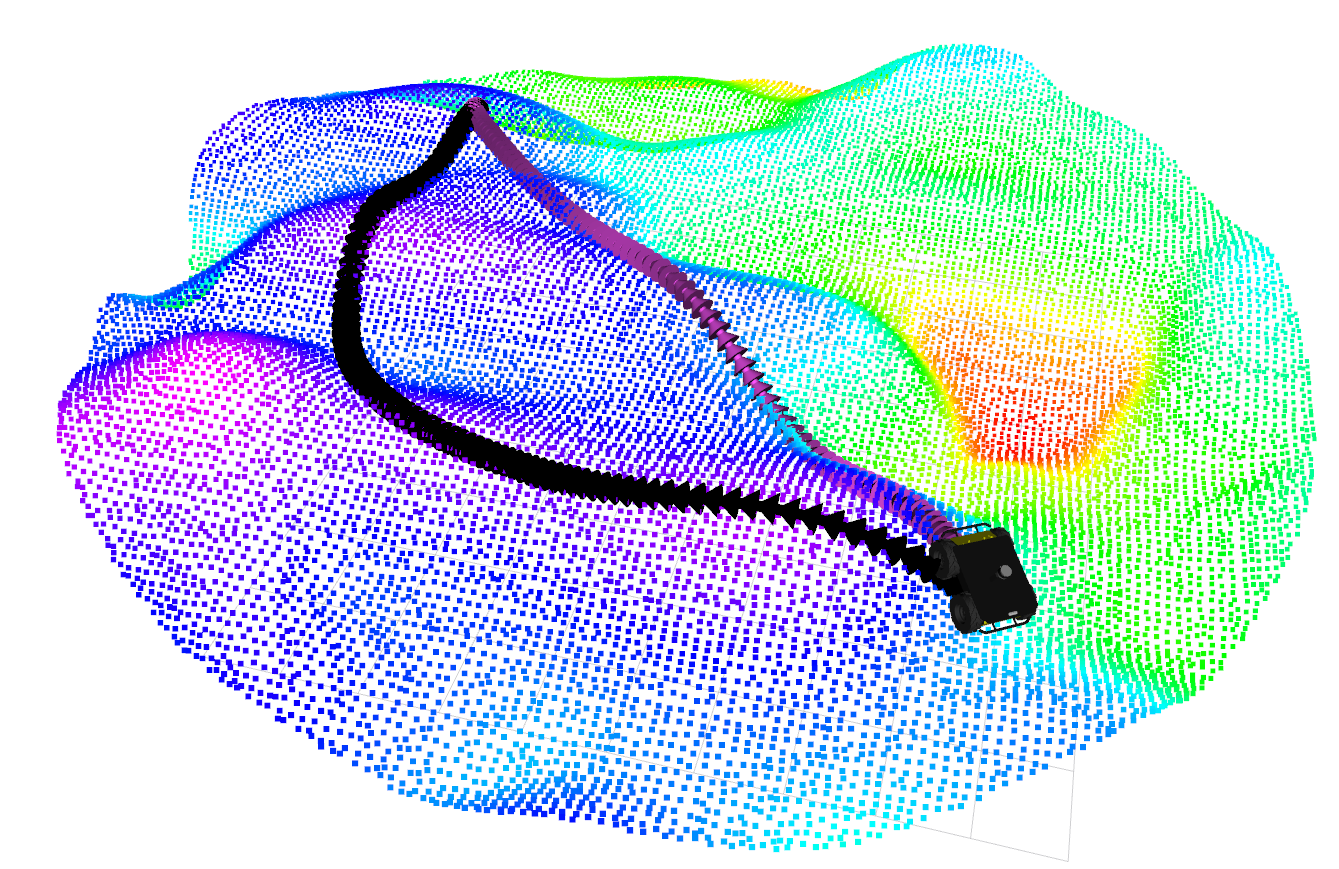}  
  \vspace{-4mm}
  \caption{\scriptsize Example 1}
  \label{fig:traj_terrain_2}
  \vspace{-1mm}
\end{subfigure}\hspace{20mm}
\begin{subfigure}{7cm}
  \includegraphics[width=\linewidth]{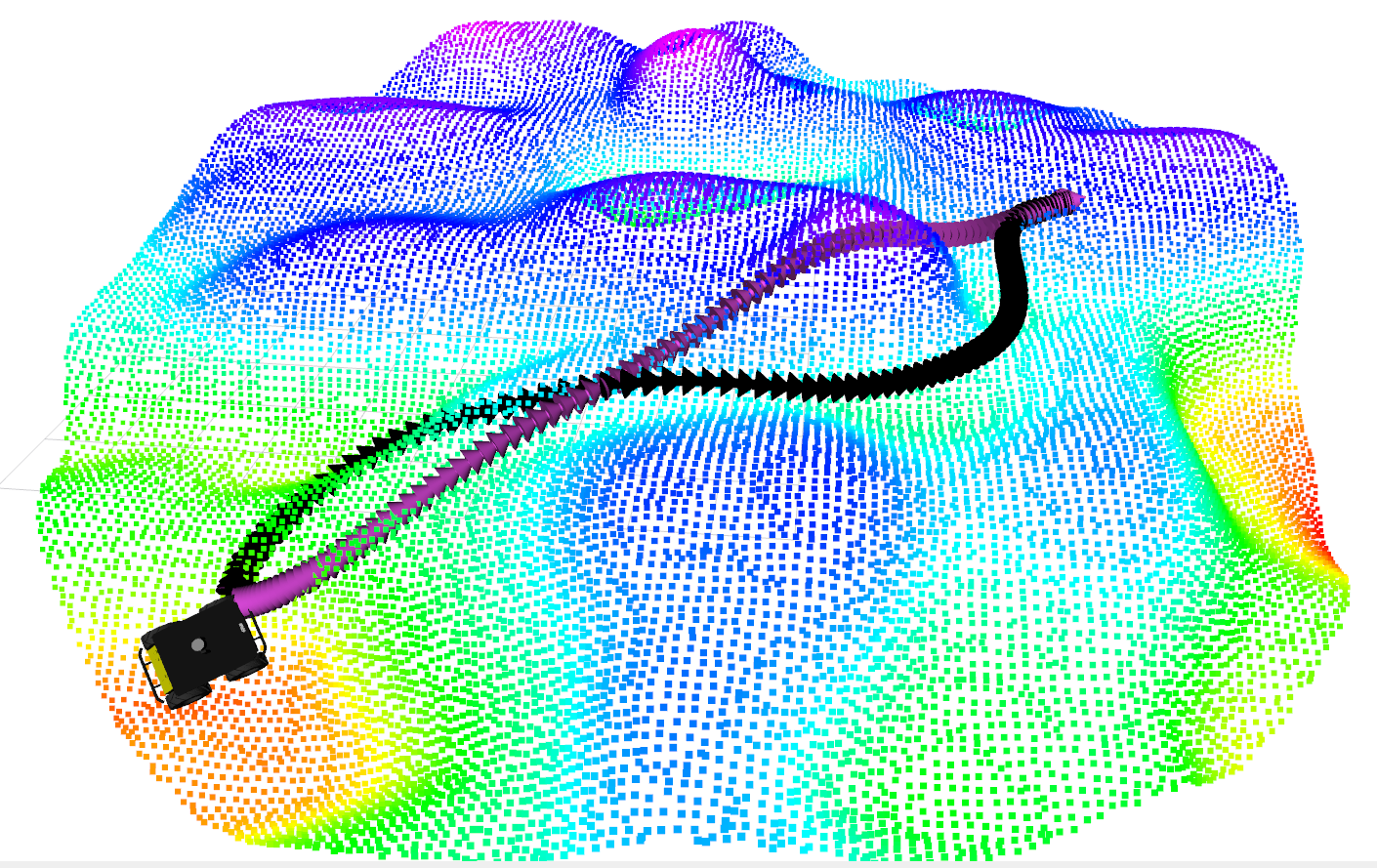}  
  \vspace{-5mm}
  \caption{\scriptsize Example 2}
  \label{fig:traj_terrain_6}
  \vspace{-1mm}
\end{subfigure}

\begin{subfigure}{7cm}
  \includegraphics[width=\linewidth]{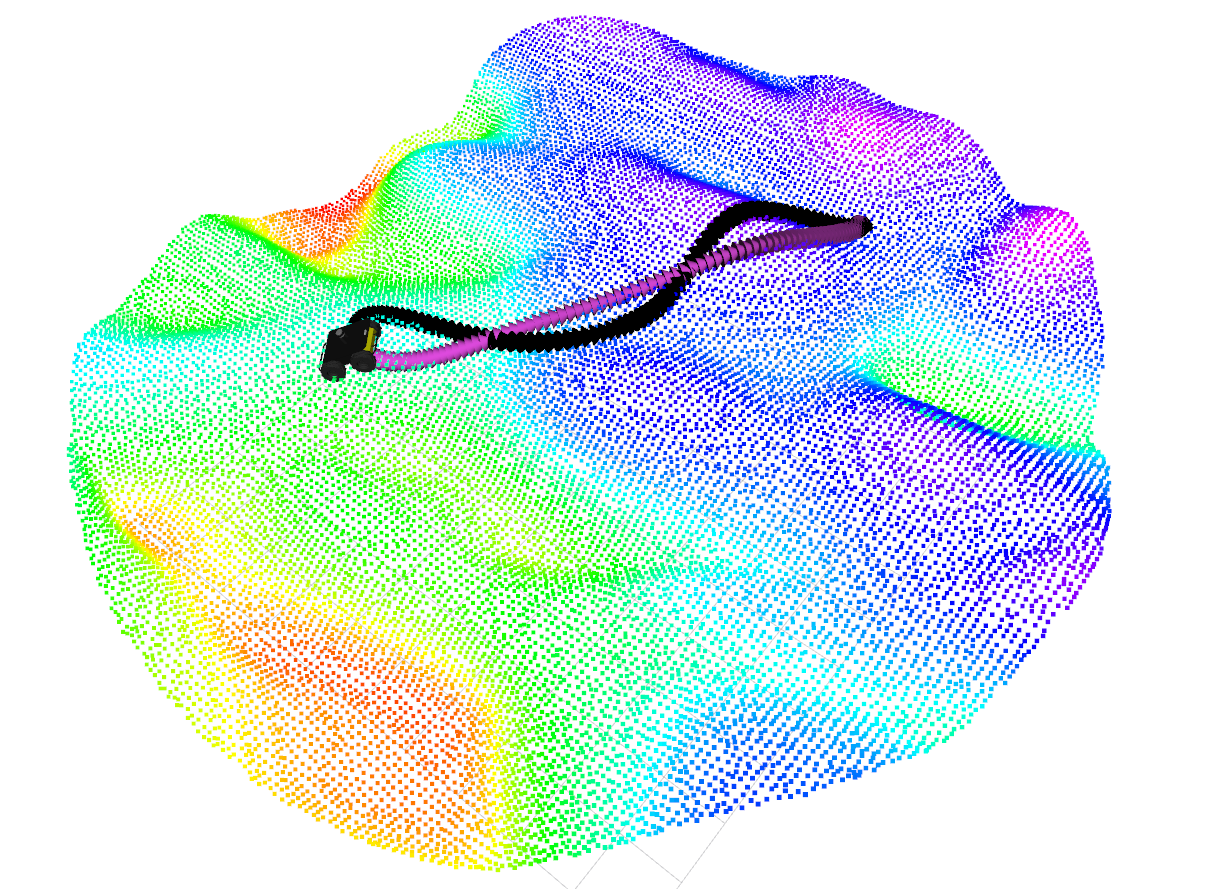}  
  \vspace{-5mm}
  \caption{\scriptsize Example 3}
  \label{fig:traj_terrain_3}
  \vspace{-1mm}
\end{subfigure}\hspace{20mm}
\begin{subfigure}{7cm}
  \includegraphics[width=\linewidth]{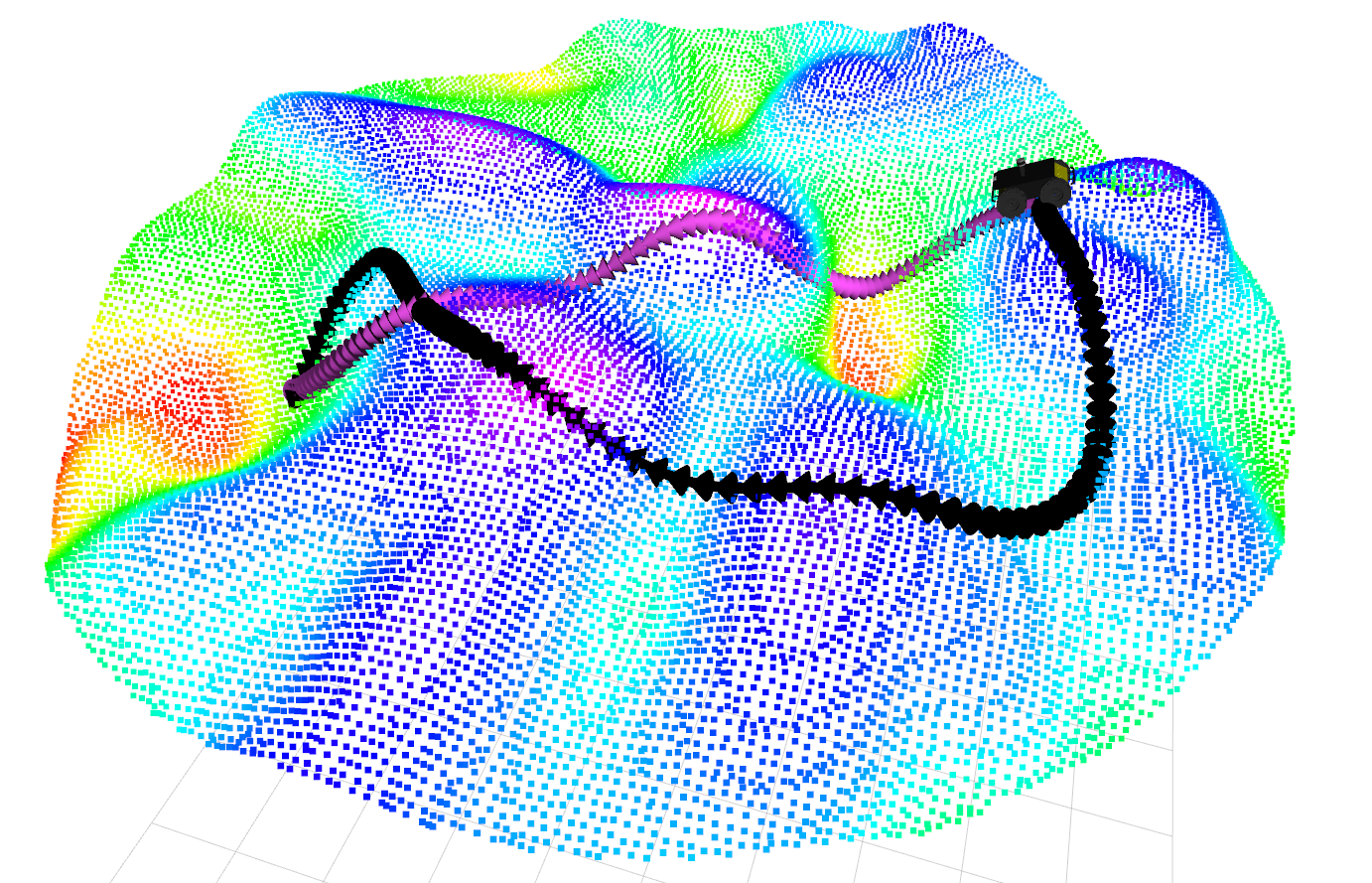}  
  \vspace{-5mm}
  \caption{\scriptsize Example 4}
  \label{fig:traj_terrain_7}
  \vspace{-1mm}
\end{subfigure}
\caption{\footnotesize Trajectories with and without stability cost. Black colour represents trajectory with stability cost and pink colour without the stability cost.    }
\label{fig:stability_comp_traj}
\end{figure*}
\section{Results}
In this section, we validate our proposed wheel-terrain interaction and pose predictor along with the bi-level optimizer through extensive simulations. We mainly focus on answering the following research questions with the help of qualitative and quantitative (statistical) results.

\begin{itemize}
    \item \textbf{Q1:} How well does our NLS based pose predictor fare vis-a-vis high-fidelity physics simulator Gazebo \cite{koenig2004design}?
    \item \textbf{Q2:} How does the stability cost affect the trajectory taken by the vehicle?
    \item \textbf{Q3:} How does our gradient descent approach compare to the state-of-the-art Cross Entropy Method (CEM) \cite{rubinstein1999cross} sampling based planner?
\end{itemize}

\begin{figure*}[!t]
\begin{subfigure}{4.2cm}
  \includegraphics[width=\linewidth]{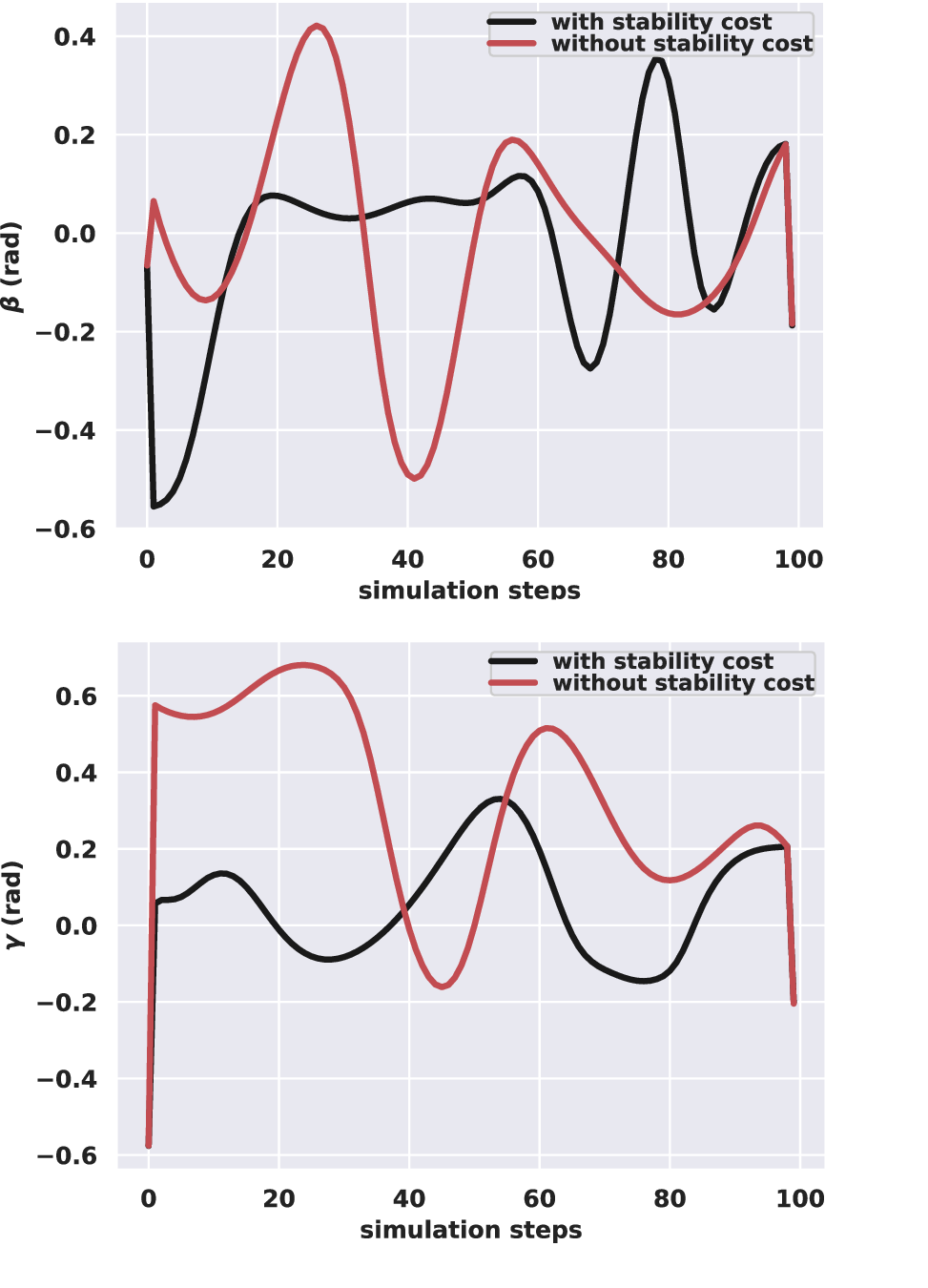}  
  \vspace{-5mm}
  \caption{\scriptsize Example 1}
  \label{fig:beta_gamma_2}
  \vspace{-1mm}
\end{subfigure}
\begin{subfigure}{4.2cm}
  \includegraphics[width=\linewidth]{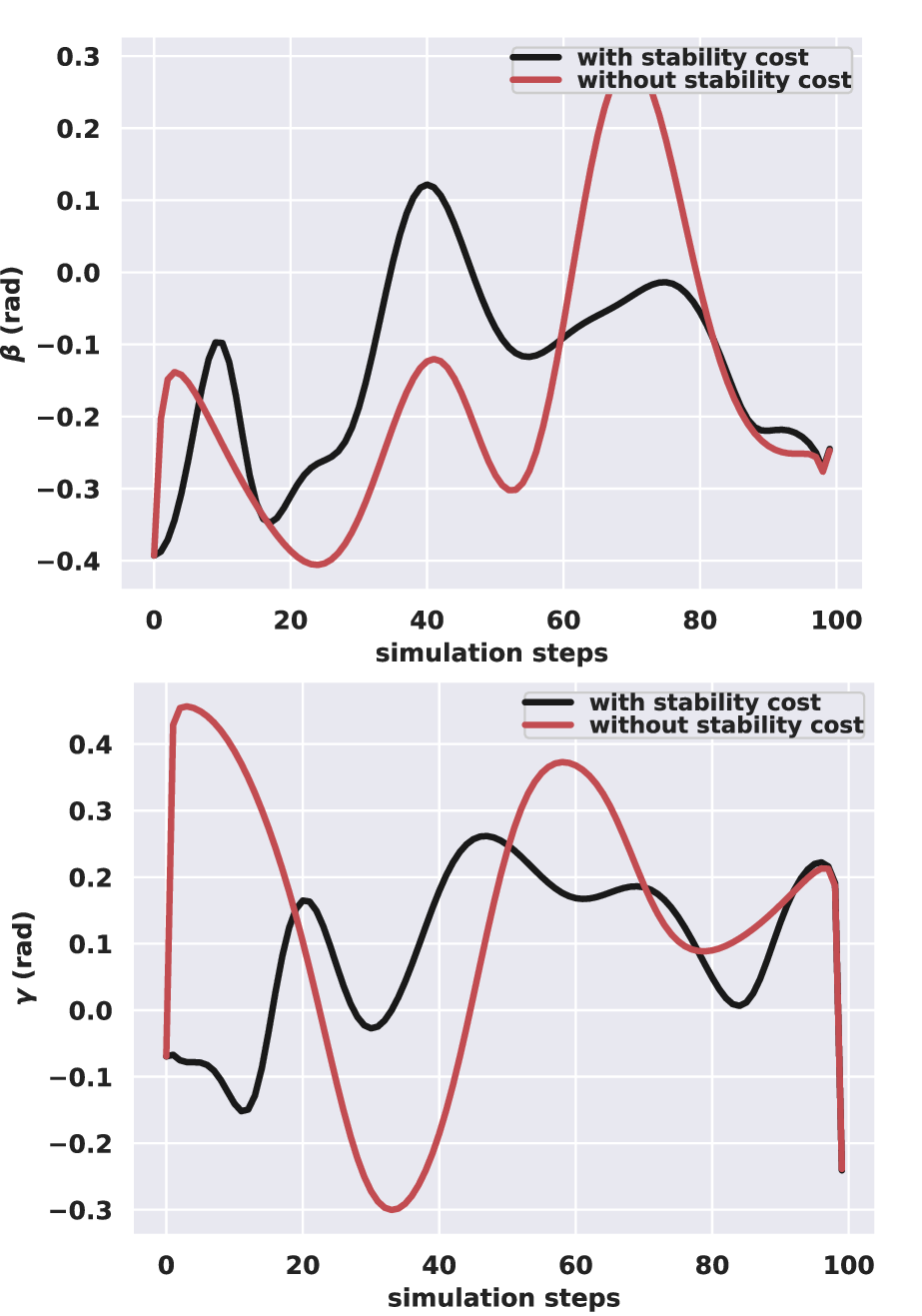}  
  \vspace{-5mm}
  \caption{\scriptsize Example 2}
  \label{fig:beta_gamma_6}
  \vspace{-1mm}
\end{subfigure}
\begin{subfigure}{4.3cm}
  \includegraphics[width=1.03\linewidth]{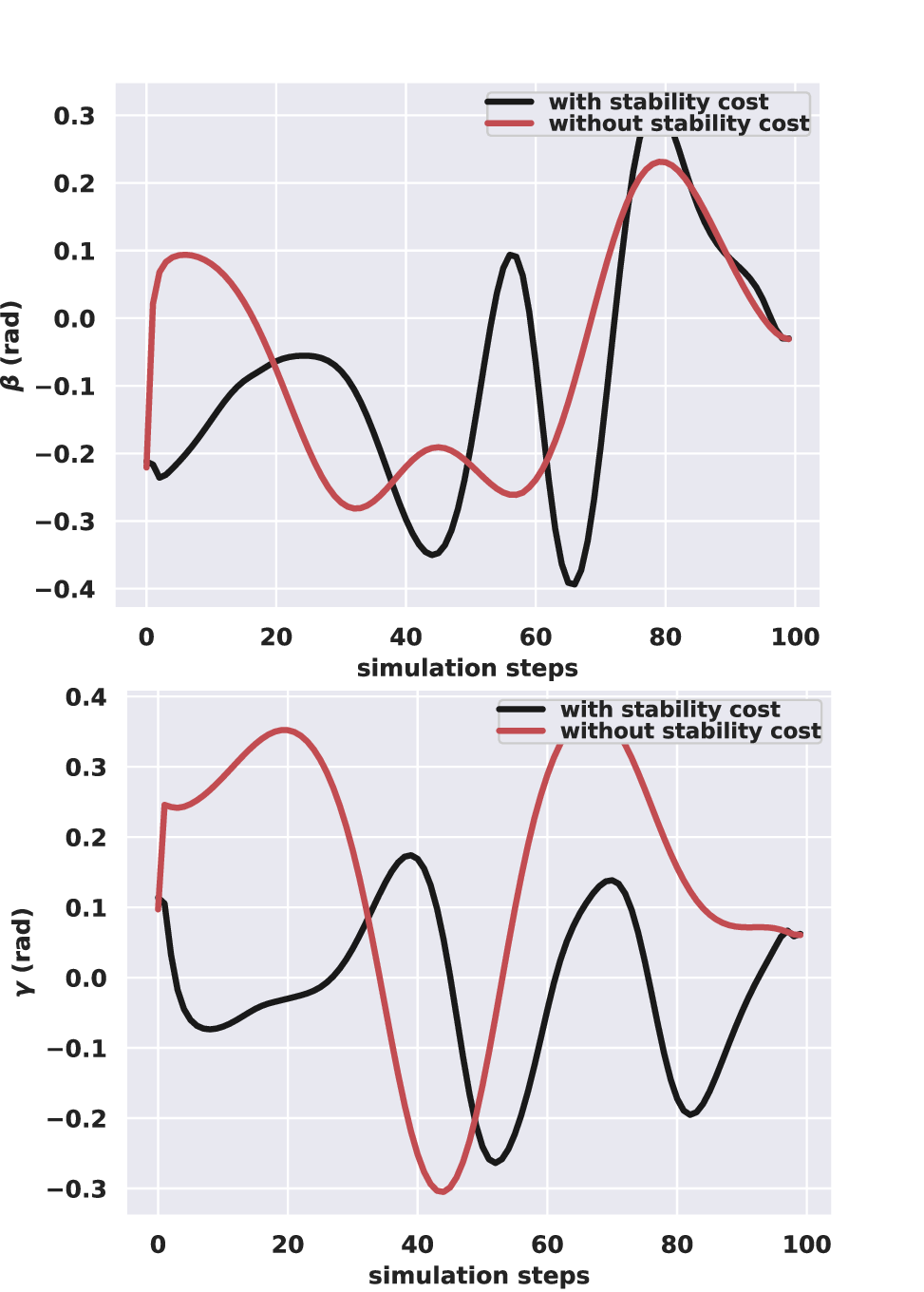}  
  \vspace{-5mm}
  \caption{\scriptsize Example 3}
  \label{fig:beta_gamma_3}
  \vspace{-1mm}
\end{subfigure}
\begin{subfigure}{4.3cm}
  \includegraphics[width=1.03\linewidth]{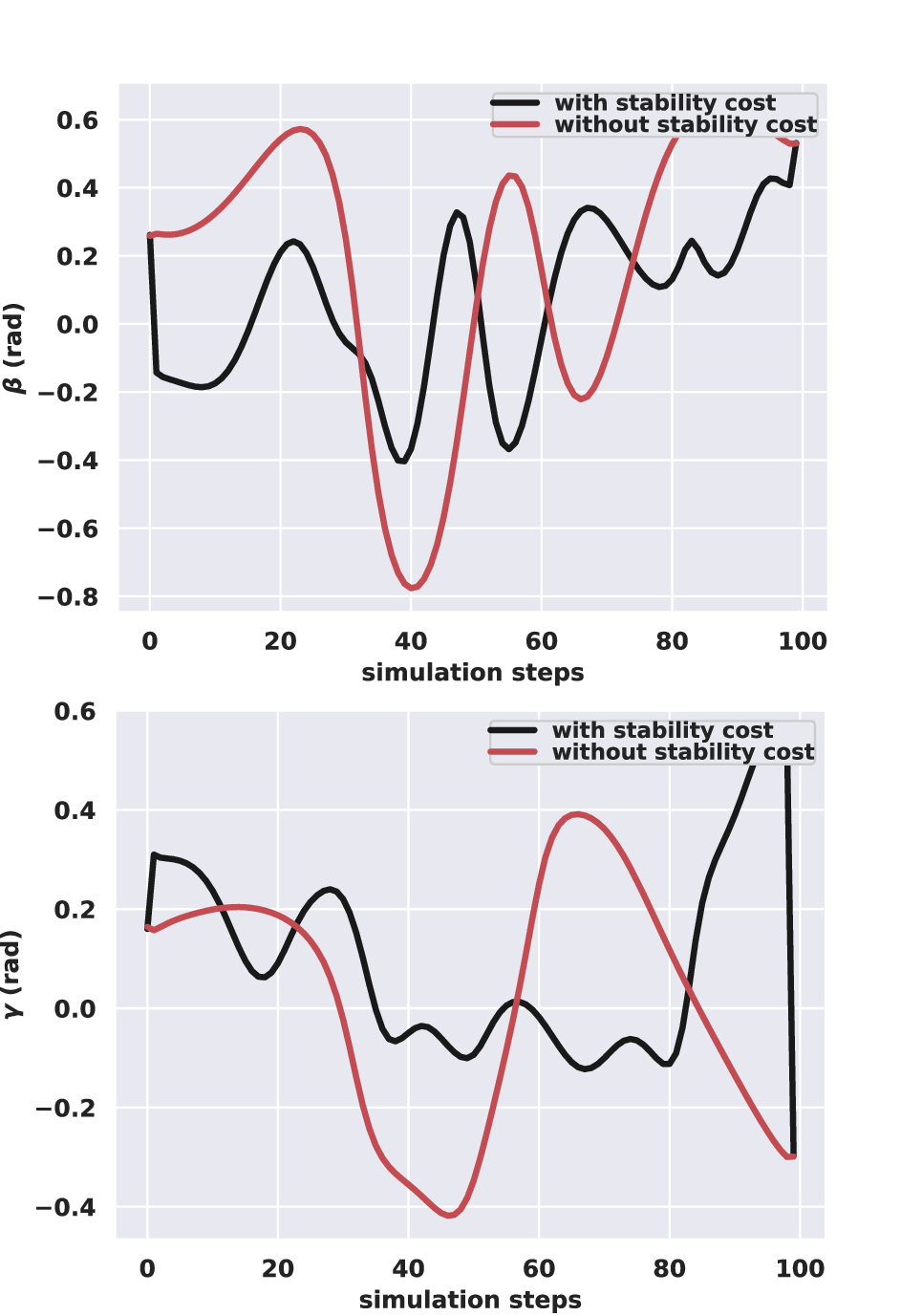}  
  \vspace{-5mm}
  \caption{\scriptsize Example 4}
  \label{fig:beta_gamma_7}
  \vspace{-1mm}
\end{subfigure}
\caption{\footnotesize Roll and pitch angle variations with and without the stability cost.}
\label{fig:stability_comp_beta_gamma}
\vspace{-5mm}
\end{figure*}

\subsection{Implementation Details}
\noindent The trajectory planner comprising of the bi-level optimization \eqref{upper_cost}-\eqref{lower_eq} and projection \eqref{eqn:projection}-\eqref{eqn:proj_ineq} is implemented in Python using the JAX \cite{jax} library for GPU-accelerated computing. The matrix $\textbf{W}$ in \eqref{param} is constructed from a $10^{th}$ order polynomial. Our simulation pipeline consists of open-loop analyses carried out with stand-alone Python code, and the closed-loop simulations were performed on the Gazebo simulator \cite{koenig2004design}. The terrains were created using Blender \cite{blender} and were converted to the required formats for the Python and Gazebo simulations. All simulations were conducted for $100$ time-steps with a sampling time of $0.2s$ on a computer with Intel i9 CPU and Nvidea RTX 3080 GPU. For modeling the terrain at each instant of time, we take a patch of radius $7m$ centered at the vehicle position. The number of Fourier frequencies $N$ in the terrain model was taken as $100$ to balance the trade-off between computation time and accuracy. 

\subsection{Validation using the Gazebo simulator}

\noindent In this section, we analyze the accuracy of our pose predictor by comparing the ground truth data obtained from Gazebo with the values obtained by solving the NLS problem \eqref{eqn:u_k} with the help of \eqref{eqn:fft_fn}. To this end, we drive the Husky robot in Gazebo on synthetic terrain models and collect the odometry data, as shown in Fig.~\ref{fig:manual_traj}. We then compare the ground truth values $z_{gt}$, $\beta_{gt}$, and $\gamma_{gt}$ from Gazebo with the predicted $z_{pred}$, $\beta_{pred}$, and $\gamma_{pred}$ from our pose optimizer. Several runs of the robot on different terrains were performed, and some qualitative results are shown in Figures \ref{fig:z_plot_1} and \ref{fig:beta_gamma_plot_1}. We can see that the predicted pose values closely match the ground truth values. The statistical results presented in Figures~\ref{fig:z_box_plot_1} and \ref{fig:beta_gamma_box_plot_1} show the minimum, maximum, and median error obtained across several runs. It can be seen that our NLS pose optimizer is highly accurate in predicting the pose of the robot on different terrains, and the error values are minimal.

\subsection{Validation of the Safe Planning on Uneven Terrain}
\noindent In this subsection, we evaluate the performance of our trajectory optimizer by considering the tip-over stability metric as an indicator of safe trajectories. Fig~\ref{fig:stability_comp_traj} shows four example trajectories generated by the planner with and without the stability cost $c_s(.)$ defined in \eqref{eqn:stability_cost}. From the figures, we can see that the trajectory computed by the planner with stability cost closely follows the terrain gradients, while the one without the stability cost cuts across hills and valleys in the terrain. The former behavior results in a safer and more stable trajectory toward the goal point, whereas the latter is prone to tipping along the way. The variations in the roll angle $\beta$ and the pitch angle $\gamma$ for the trajectories are plotted in Fig.~\ref{fig:stability_comp_beta_gamma}. It can be seen that the magnitude of the angles is higher for the trajectory without the stability cost. Also, there are sharp variations in the angles, which means that the vehicle is traversing through a very uneven terrain. Sudden changes in the vehicle's orientation may lead to the loss of stability. With the stability cost in effect, the magnitude and variations in the angles become lower, and the vehicle becomes more stable.

Next, we present a statistical result in Fig.~\ref{fig:stability_tip_comp}. Here, we also highlight the significance of the weighting term $w_{\theta}$ in the stability cost \ref{eqn:stability_cost}. Several trajectories were simulated on different terrains, and the average worst-case tip-over angles were recorded. It can be seen that for the trajectories without stability cost, the value is lower. In the case of tip-over angle, higher positive values are preferable, as shown in \eqref{eqn:tip_over_stability_criteria}. The values with stability cost are far higher than those without stability cost, which confirms that the trajectories taken by the vehicle with stability cost have a higher safety margin. It can also be seen that as the weight $w_{\theta}$ increased from 0.05 to 0.2, the values of the angles increased, which in turn increased the tip-over stability margin.

\subsection{Comparison with CEM sampling-based planner}
\noindent This section investigates the importance of differentiability of the wheel-terrain interaction/pose predictor model. If the implicit gradients are not available, then one can use a black-box sampling-based optimizer such as CEM. However, it will require an order of magnitude larger calls of the NLS pose predictor. In fact, it may not even be possible to solve our bi-level optimizer through a sampling-based approach on resource-constrained hardware. Nevertheless, in this section, our aim is to show that our gradient-based approach is competitive with a more computationally demanding approach. The results are summarized in  Fig~\ref{fig:stability_tip_comp_cem_grad}, which shows the statistical results of the average worst-case tip-over angle obtained with our gradient-based approach and CEM. Two different batch sizes (100 and 20) were used for CEM.  The results show that our gradient descent-based planner performs as well as a CEM planner with batch size 100 and is slightly better than a CEM with batch size 20. The average computation time for our approach was $0.42s$, whereas the CEM with batch sizes 100 and 20 took $2.12s$ and $1.03s$, respectively. It is also to be noted that the CEM planners are memory-hungry; thus, implementation on onboard computers might be difficult, if at all possible. 



\begin{figure}
\begin{subfigure}{4.2cm}
  \includegraphics[width=\linewidth]{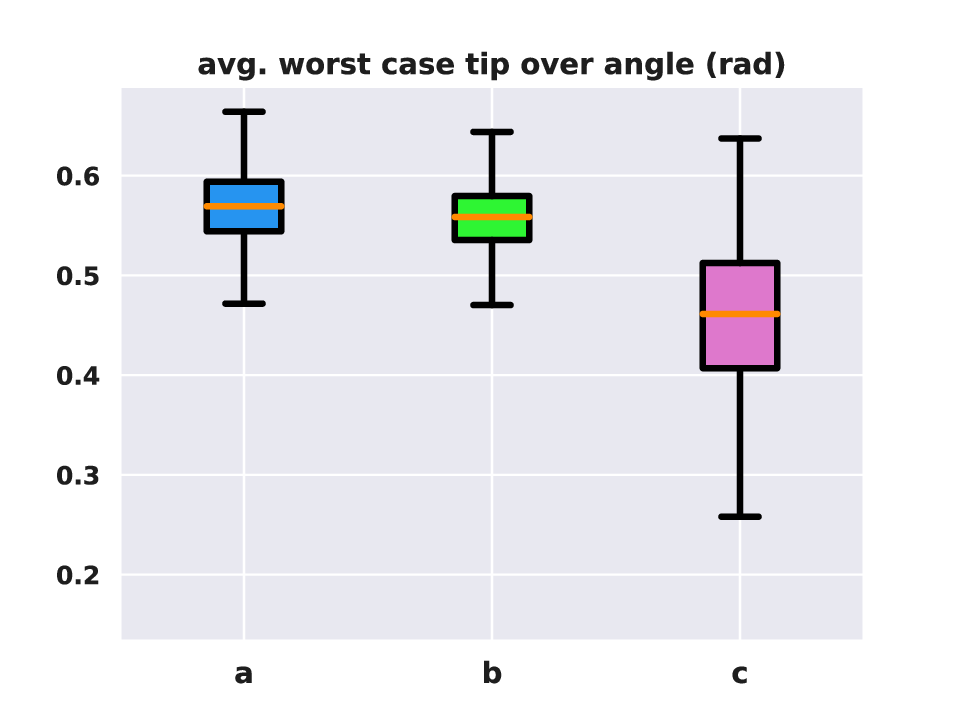}  
  \vspace{-5mm}
  \caption{\scriptsize a: with $c_s$, $w_{\theta}=0.2$, b: with $c_s$, $w_{\theta}=0.05$, c: without $c_s$}
  \label{fig:stability_tip_comp}
  \vspace{-1mm}
\end{subfigure}
\begin{subfigure}{4.2cm}
  \includegraphics[width=\linewidth]{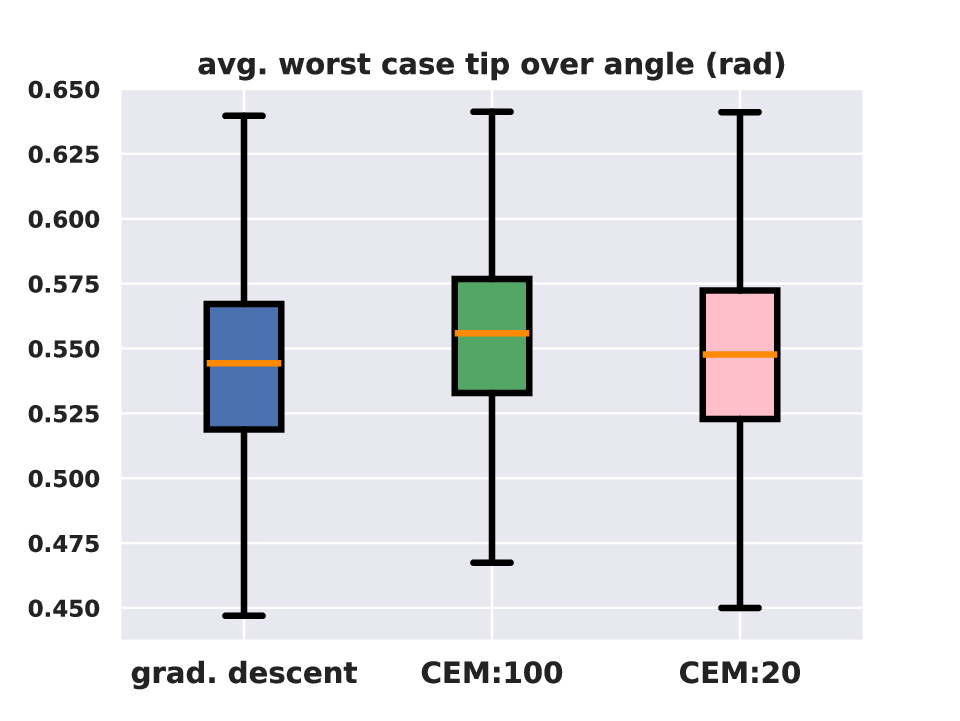}  
  \vspace{-5mm}
  \caption{\scriptsize gradient descent compared with CEM with batch sizes 100, 20.}
  \label{fig:stability_tip_comp_cem_grad}
  \vspace{-1mm}
\end{subfigure}
\caption{\footnotesize (a) Comparison of the effect of the stability cost, $c_s$. (b) Comparison between gradient descent and CEM.}
\label{fig:stability_comp_tip}
\vspace{-0.5cm}
\end{figure}

\section{conclusions} \label{sec:conclusions}
We presented, for the first time, a differentiable wheel-terrain interaction/pose predictor model that is derived purely from the first principles and yet generalizes to arbitrary terrains. The fidelity of our model was shown to be close to that of a state-of-the-art physics engine. This opens up exciting new possibilities. For example, our pose predictor can be used as a differentiable and parallelizable world model within trajectory optimization or even reinforcement learning pipelines. We presented one such approach based on bi-level optimization, wherein we leveraged implicit gradients from the NLS solver to efficiently compute the optimal trajectory. The proposed optimizer can accommodate arbitrary stability metrics that depend on the vehicle's pose on uneven terrains. As an example, we used the force-angle measure and demonstrated stable trajectory planning on uneven terrains. Our future endeavors are focused on extending the NLS-based approach to handle wheel-terrain dynamics.
\bibliographystyle{IEEEtran}
\bibliography{ref}
\end{document}